\documentclass[letterpaper]{article} 
\usepackage{aaai25}  
\usepackage{times}  
\usepackage{helvet}  
\usepackage{courier}  
\usepackage[hyphens]{url}  
\usepackage{graphicx} 
\usepackage{natbib}  
\usepackage{caption} 
\usepackage{algorithm}
\usepackage{algorithmic}

\usepackage[utf8]{inputenc} 
\usepackage[T1]{fontenc}    
\usepackage{url}            
\usepackage{booktabs}       
\usepackage{amsfonts}       
\usepackage{nicefrac}       
\usepackage{microtype}      
\usepackage{xcolor}         

\usepackage{microtype}
\usepackage{graphicx}
\usepackage{subcaption}
\usepackage[export]{adjustbox}

\usepackage{booktabs} 

\usepackage{amsmath}
\usepackage{amssymb}
\usepackage{mathtools}
\usepackage{amsthm}

\usepackage[capitalize,noabbrev]{cleveref}

\theoremstyle{plain}
\newtheorem{theorem}{Theorem}[section]

\theoremstyle{definition}

\theoremstyle{remark}

\usepackage{multirow}
\newcommand{\bfA}{{\bf A}}

\newcommand{\bfD}{{\bf D}}

\newcommand{\bfF}{{\bf F}}

\newcommand{\bfI}{{\bf I}}
\newcommand{\bfJ}{{\bf J}}
\newcommand{\bfK}{{\bf K}}
\newcommand{\bfL}{{\bf L}}
\newcommand{\bfM}{{\bf M}}

\newcommand{\bfV}{{\bf V}}
\newcommand{\bfW}{{\bf W}}
\newcommand{\bfX}{{\bf X}}

\newcommand{\bfZ}{{\bf Z}}

\newcommand{\bfx}{{\bf x}}

\newcommand{\ie}{i.e., }

\definecolor{third}{HTML}{ef8700}  
\definecolor{first}{HTML}{785EF0} 
\definecolor{second}{HTML}{DC267F}  
\newcommand{\one}[1]{\textcolor{third}{\bf#1}}
\newcommand{\two}[1]{\textcolor{first}{\bf#1}}
\newcommand{\three}[1]{{\bf#1}}
\newcommand{\four}[1]{{\bf#1}}

\definecolor{transfercolor}{HTML}{d55e00} 
\definecolor{graphpropcolor}{HTML}{029e73} 
\definecolor{lrbcolor}{HTML}{0173b2} 
\usepackage[inline]{enumitem}

\usepackage{newfloat}
\usepackage{listings}
\DeclareCaptionStyle{ruled}{labelfont=normalfont,labelsep=colon,strut=off} 
\floatstyle{ruled}
\newfloat{listing}{tb}{lst}{}
\floatname{listing}{Listing}
\title{On Oversquashing in Graph Neural Networks \\ Through the Lens of Dynamical Systems}
\author{
Alessio Gravina\equalcontrib\textsuperscript{\rm 1}, Moshe Eliasof\equalcontrib\textsuperscript{\rm 2}, Claudio Gallicchio\textsuperscript{\rm 1}, Davide Bacciu\textsuperscript{\rm 1}, Carola-Bibiane Sch\"onlieb\textsuperscript{\rm 2}
}
\affiliations{
\textsuperscript{\rm 1}Department of Computer Science, University of Pisa, Pisa, Italy\\
\textsuperscript{\rm 2}Department of Applied Mathematics and
Theoretical Physics, University of Cambridge, Cambridge, United Kingdom\\
alessio.gravina@di.unipi.it, me532@cam.ac.uk, \{claudio.gallicchio, davide.bacciu\}@unipi.it, cbs31@cam.ac.uk

}

\usepackage{bibentry}

\begin{document}

\maketitle

\begin{abstract}
A common problem in Message-Passing Neural Networks is oversquashing -- the limited ability to facilitate effective information flow between distant nodes. Oversquashing is attributed to the exponential decay in information transmission as node distances increase. This paper introduces a novel perspective to address oversquashing, leveraging dynamical systems properties of global and local non-dissipativity, that enable the maintenance of a constant information flow rate. We present SWAN, a uniquely parameterized GNN model with antisymmetry both in space and weight domains, as a means to obtain non-dissipativity. Our theoretical analysis asserts that by implementing these properties, SWAN offers an enhanced ability to transmit information over extended distances. Empirical evaluations on synthetic and real-world benchmarks that emphasize long-range interactions validate the theoretical understanding of SWAN, and its ability to mitigate oversquashing.
\end{abstract}

%

\section{Introduction}
\label{sec:intro}
A critical issue that limits Message-Passing Neural Networks (MPNNs) \cite{gilmer2017neural}, a class of GNNs, is the oversquashing problem \cite{alon2021oversquashing, diGiovanniOversquashing}. In the oversquashing scenario, the capacity of MPNNs to transmit information between nodes exponentially decreases as their distance increases, which imposes challenges on modeling long-range interactions, which are often necessary for real-world tasks \cite{dwivedi2022LRGB}.
At the same time, it has been shown in \cite{HaberHolthamRuthotto2017, chen2018neural} that neural networks can be interpreted as the discretization of ordinary differential equations (ODEs).  Building on these observations and understandings, similar concepts were utilized to forge the field of differential-equations inspired GNNs (DE-GNNs), as shown in \cite{poli, chamberlain2021grand} and subsequent works. Through this view, it is possible to design GNNs with strong inductive biases, such as smoothness \cite{chamberlain2021grand}, energy preservation \cite{eliasof2021pde, rusch2022graph}, node-wise non-dissipativity \cite{gravina2022anti}, and more.

In this paper, we are interested in addressing the oversquashing problem in a principled manner, accompanied by theoretical understanding through the prism of DE-GNNs. Current literature offers 
several approaches to mitigate oversquashing, such as adding a virtual global node \cite{gilmer2017neural, cai2023connection}, graph rewiring \cite{gasteiger_diffusion_2019,topping2022understanding,karhadkar2023fosr}, as well as using graph transformers \cite{ dwivedi2021generalization, rampasek2022GPS}. Some of the methods above focus on providing ad-hoc mechanisms for the network to reduce oversquashing. Instead, in this paper, we are interested in theoretically understanding and improving the information propagation capacity of the network, thereby mitigating oversquashing. Specifically, we will show that it is possible to design GNNs with a constant flow of information, unlike typical diffusion-based methods. To this end, we take inspiration from the recent ADGN \cite{gravina2022anti}, a 
non-dissipative GNN. In ADGN, it was shown that by incorporating an antisymmetric transformation to the learned channel-mixing weights, it is possible to obtain a \emph{locally}, i.e., node-wise, non-dissipative behavior. Here, we propose \textbf{SWAN} (\textbf{S}pace-\textbf{W}eight \textbf{AN}tisymmetry), a novel GNN model that is both \emph{globally} (\ie graph-wise) 
and 
\emph{locally} (\ie node-wise) non-dissipative, achieved by space and weight antisymmetric parameterization. To understand the behavior of SWAN and its effectiveness in mitigating oversquashing, we propose a \emph{global}, i.e., graph-wise, analysis, and show that compared to ADGN \cite{gravina2022anti}, our SWAN is both globally and locally 
non-dissipative. The immediate implication of this property is that SWAN is guaranteed to have a constant information flow rate, thereby mitigating oversquashing. Such a property is visualized in Figure~\ref{fig:propagation}, where SWAN shows improved capacities of propagating information across the graph, with respect to diffusion and local non-dissipative approaches. To complement our theoretical analysis, we experiment with several synthetic and real-world long-range benchmarks.

\begin{figure*}[t]
    \centering
  \begin{subfigure}{0.18\textwidth}
    \centering \includegraphics[width=0.75\linewidth]{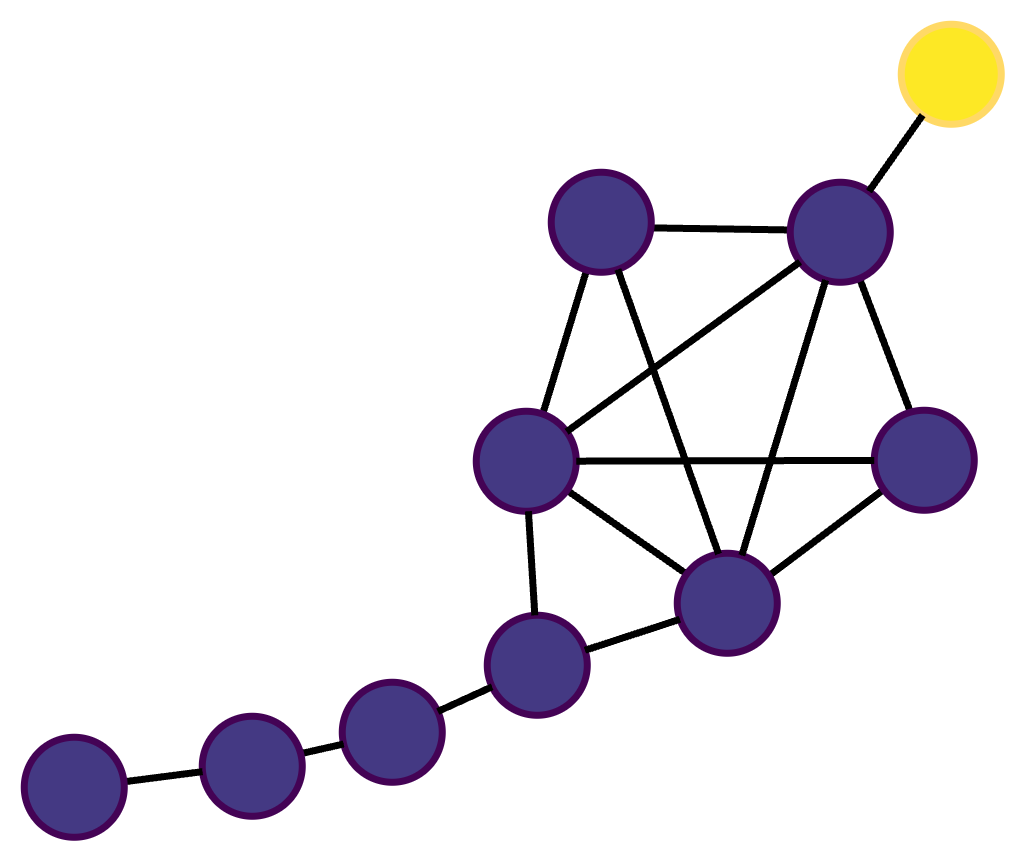}
    \caption{Source\vspace{0.32cm}}
    \label{fig:sub1}
  \end{subfigure}\hspace{0.2cm}
   \begin{subfigure}{0.18\textwidth}
    \centering \includegraphics[width=0.75\linewidth]{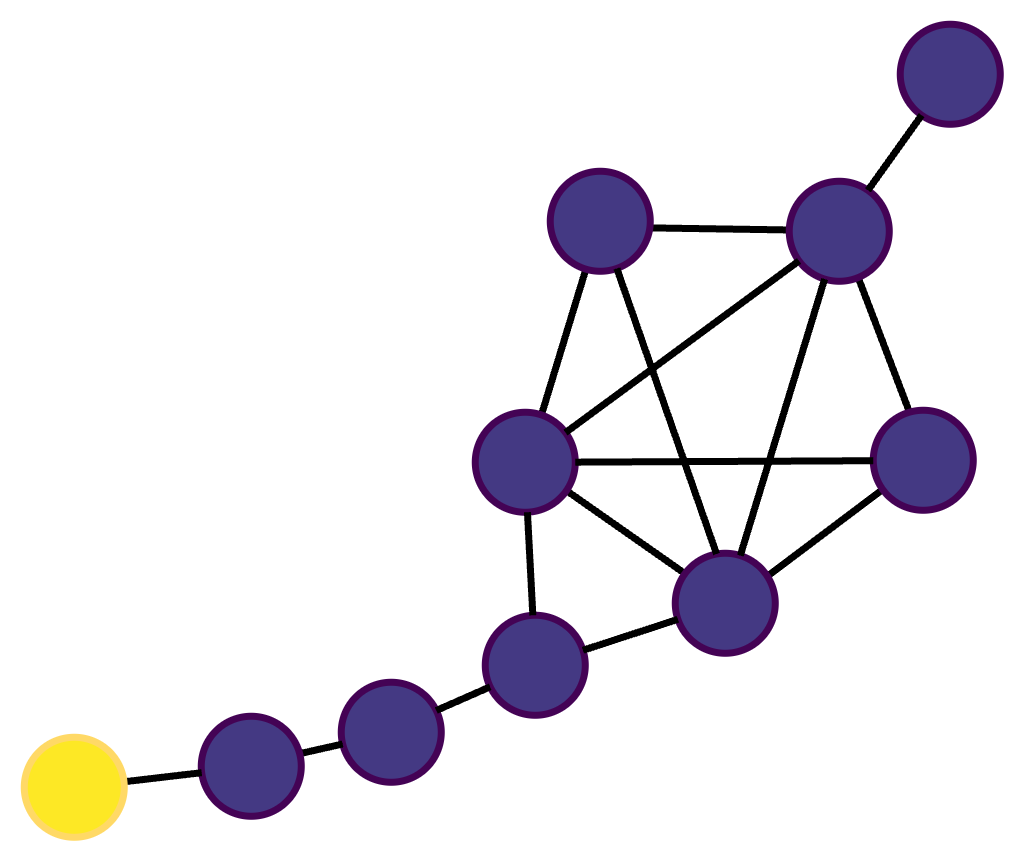}
    \caption{Target\vspace{0.32cm}}
    \label{fig:sub2}
  \end{subfigure}\hspace{0.2cm}
   \begin{subfigure}{0.18\textwidth}
    \centering \includegraphics[width=0.75\linewidth]{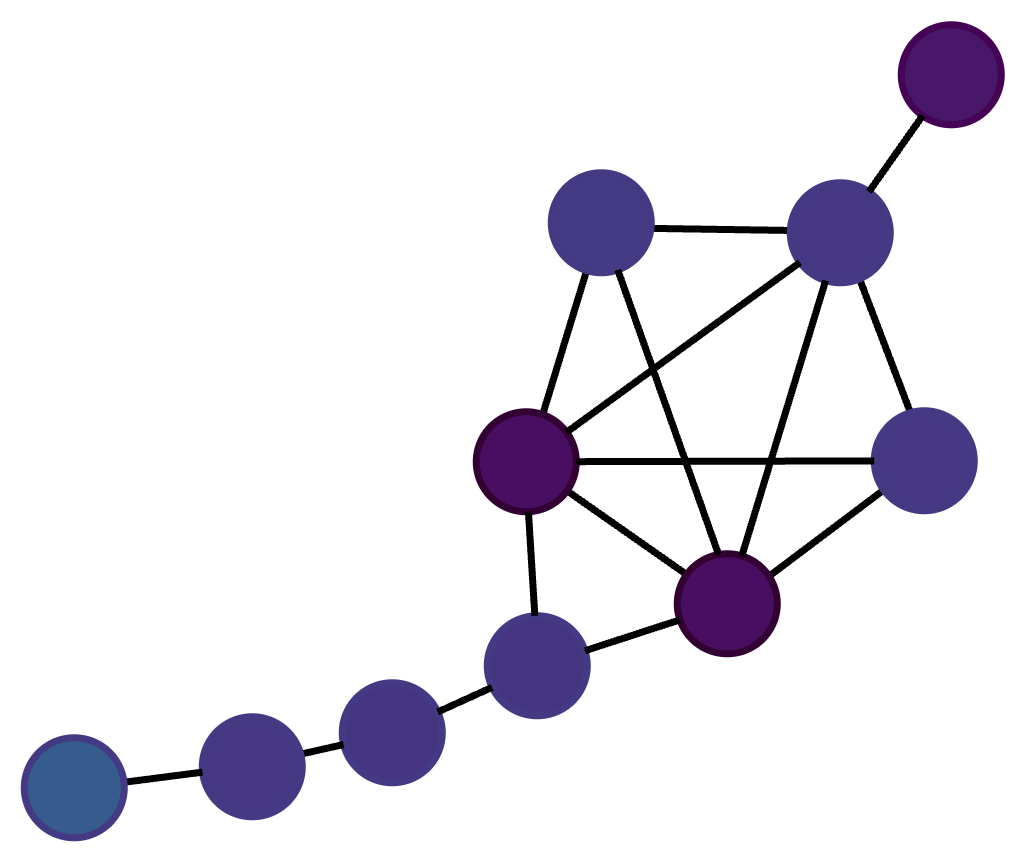}
    \caption{Diffusion\vspace{0.32cm}}
    \label{fig:sub3}
  \end{subfigure}\hspace{0.2cm}
   \begin{subfigure}{0.18\textwidth}
    \centering \includegraphics[width=0.75\linewidth]{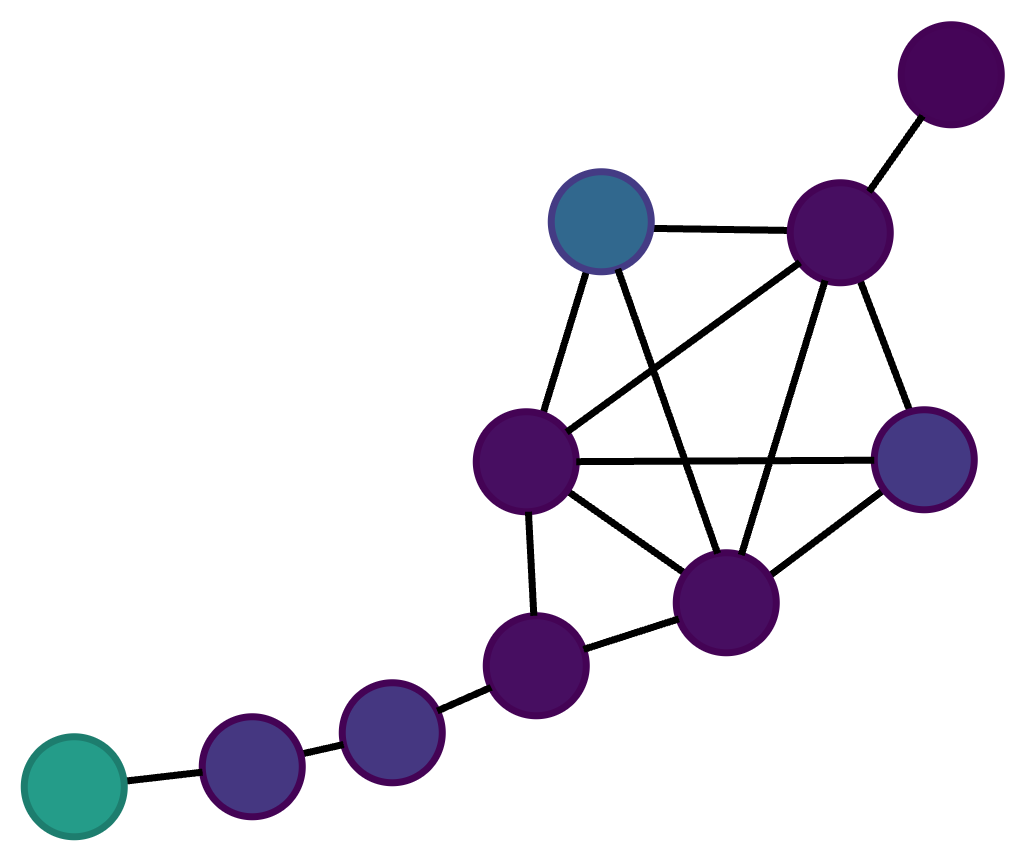}
    \caption{Local \\Non-Dissipativity }
    \label{fig:sub4}
  \end{subfigure}\hspace{0.2cm}
   \begin{subfigure}{0.18\textwidth}
    \centering \includegraphics[width=0.75\linewidth]{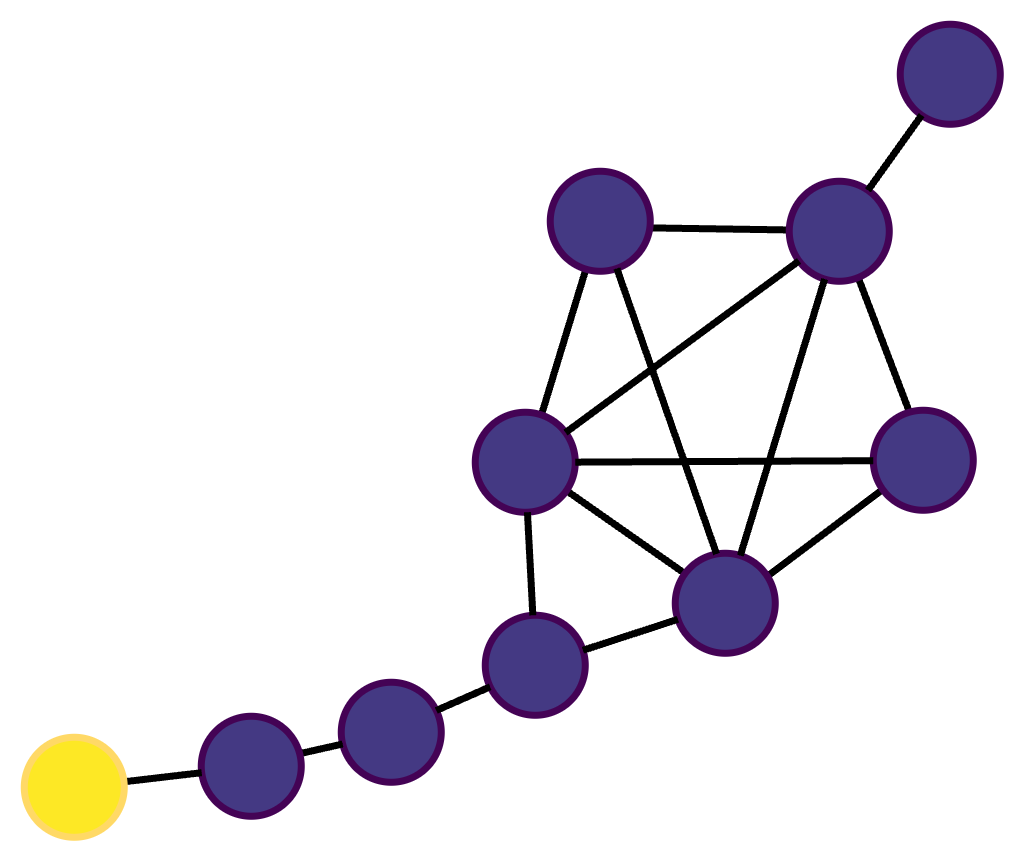}
    \caption{Local and Global Non-Dissipativity}
    \label{fig:sub5}
   \end{subfigure}
  \caption{An illustration of the ability of Global and Local Non-Dissipativity in SWAN (e) to propagate information to distant nodes, from source (a) to the target (b). Other dynamics, such as diffusion (c) cannot achieve this behavior, while Local Non-Dissipativity (d) offers a limited effect.}
    \label{fig:propagation}
\end{figure*}

\noindent \textbf{Main Contributions.} We present the following contributions. (i) A novel graph perspective theoretical analysis of the stability and non-dissipativity of antisymmetric DE-GNNs, providing a general design principle for introducing non-dissipativity as an inductive bias in any DE-GNN model. (ii) We propose SWAN, a space and weight antisymmetric DE-GNN with a constant information flow rate, and an increased distant node interaction sensitivity. (iii) We experimentally verify our theoretical understanding of SWAN, on both synthetic and real-world datasets. 
Our experiments show score improvements of up to 117\% in synthetic datasets, as well as competitive results on real-world benchmarks compared to existing methods, from MPNNs to other DE-GNNs, and graph transformers -- highlighting the importance of global and local non-dissipative behavior offered by SWAN.

\section{Preliminaries}
\label{sec:preliminaries}
We provide an overview of oversquashing problem in GNNs, followed by a brief discussion of DE-GNNs and antisymmetric weights in DE-GNNs. We then provide a mathematical background, which is thoroughly used in this paper to analyze and understand the properties of SWAN.

\subsection{Oversquashing}
\emph{Oversquashing}~\cite{alon2021oversquashing} refers to the shortcoming of a GNN when transferring information between distant nodes, and is a common problem in Message-Passing Neural Networks (MPNNs). 
Because standard MPNNs update node states by aggregating neighborhood information, oversqashing is amplified as node distances increase, hampering the ability of MPNNs to model complex behaviors that require long-range interactions. Namely, to allow a node to receive information from $k$-hops distant nodes, an MPNN must employ at least $k$ layers, otherwise, it will suffer from \emph{under-reaching}, because the two nodes are too far to interact with each other. However, the stacking of multiple message-passing layers also causes each node to receive an exponentially growing amount of information, as multiple hops are considered. This exponential growth of information, combined with the finite and fixed number of channels (features), can lead to loss of information and reduced long-range effectiveness.

\subsection{GNNs Inspired by Differential-Equations}
\label{sec:degnns}
\textbf{Notation.}
We consider a graph ${\mathcal{G}=(\mathcal{V}, \mathcal{E})}$ as a system of interacting entities, referred to as nodes, where $\mathcal{V}$ is a set of
 $n$ nodes, and $\mathcal{E} \subseteq \mathcal{V} \times \mathcal{V}$ is a set of $m$ edges. We define the neighborhood of the $u$-th node as the set of nodes directly attached to it, i.e., $\mathcal{N}_u = \{v | (v,u)\in\mathcal{E} \}$. 
The $u$-th node is associated with a (possibly time-dependent) hidden feature vector 
 $\mathbf{x}_u(t) \in \mathbb{R}^d$ with $d$ features, which provides a representation of the node at time $t$ in the system. The term $\mathbf{X}(t) = [\mathbf{x}_0(t), \ldots, \mathbf{x}_{n-1}(t)]^\top$ is an $n \times d$  matrix that represents the nodes state (features) at time $t$. 
 
\noindent \textbf{DE-GNNs.} The core idea of 
DE-GNNs is to view GNNs as the discretization of the following ODE:
\begin{small}
\begin{equation}
\label{eq:basicODE}
{\frac {\partial \mathbf{x}_u(t)}{\partial t}} = s\left(\{\mathbf{x}_v(t)\}_{v\in\mathcal{N}_u}; \mathcal{G}\right),
\end{equation}
\end{small}
where $\mathbf{x}_u(t=0) = \bfx^{(0)}$ and $s\left(\{\mathbf{x}_v(t)\}_{v\in\mathcal{N}_u}; \mathcal{G}\right)$ is a spatial aggregation function that depends on the graph $\mathcal{G}$ and the node features $\mathbf{x}_u(t)$. Specifically, it is common to implement $s\left(\{\mathbf{x}_v(t)\}_{v\in\mathcal{N}_u}; \mathcal{G}\right)$ with graph diffusion,  combined with a channel mixing operator implemented by a multilayer perceptron (MLP). Some examples of such methods were proposed in \cite{chamberlain2021grand, eliasof2021pde, choi2022gread}, 
and others. 

\noindent \textbf{Antisymmetric Weights in DE-GNNs.}
Learned antisymmetric \emph{weights} were studied in ADGN \cite{gravina2022anti},  summarized below for completeness:
\begin{small}
\begin{equation}
\label{eq:adgnModel}
    {\frac{\partial\mathbf{x}_u(t)}{\partial t}} = \sigma\Bigl((\mathbf{W}-\mathbf{W}^\top)\mathbf{x}_u(t) + \Phi(\{\mathbf{x}_v\}_{v\in\mathcal{N}_u}, \mathbf{V}) \Bigr),
\end{equation}    
\end{small}
where $\bfW, \ \bfV \in \mathbb{R}^{d\times d}$ are learnable weights, $\Phi(\{\mathbf{x}_v\}_{v\in\mathcal{N}_u}, \mathbf{V})$ is any permutation invariant neighborhood aggregation function, and $\sigma$ is an activation function. 
The main theoretical property of ADGN \cite{gravina2022anti} is that it allows a \emph{stable} and \emph{non-dissipative} propagation from a \emph{local}, node-wise perspective. 

\subsection{Mathematical Background}
\label{sec:math}
While \citet{alon2021oversquashing} did not mathematically define oversquashing, recent works \cite{topping2022understanding, diGiovanniOversquashing} associate it with exponentially declining node embedding sensitivity. Building on this perspective, we connect oversquashing to non-dissipative dynamical systems, leading to our SWAN. Specifically, we are interested in studying the \emph{stability} and \emph{non-dissipativity} propagation of information in DE-GNNs. 
Therefore, we follow the analysis techniques presented in  \cite{chang2018antisymmetricrnn, gravina2022anti} and focus on analyzing the sensitivity of an ODE solution with respect to its initial condition, \ie 
\begin{small}
    \begin{equation}
    \label{eq:appendix1}
   \frac{d}{dt}\left(\frac{\partial \mathbf{x}(t)}{\partial \mathbf{x}(0)}\right) = \mathbf{J}(t) \frac{\partial \mathbf{x}(t)}{\partial\mathbf{x}(0)}.
\end{equation}
\end{small}

We now present an overview of the various outcomes that can arise from this analysis. We follow the results and assumptions from \cite{chang2018antisymmetricrnn, gravina2022anti}, and consider the Jacobian, $\mathbf{J}(t)$ in \cref{eq:appendix1} to not change significantly over time. We assume this condition in \Cref{thm:swanConstantRate,thm:diffusionExpDecay}. In Appendix I, we provide a discussion of the justification as well as numerical verification of our assumption. Given that, 
we can apply results from autonomous differential equations \cite{AscherPetzoldODEs} and solve \cref{eq:appendix1} analytically as follows:
\begin{small}
    \begin{equation}
   \frac{\partial \mathbf{x}(t)}{\partial \mathbf{x}(0)} = e^{t \mathbf{J}} = \mathbf{T} e^{t \mathbf{\Lambda}}\mathbf{T}^{-1}
   = \mathbf{T} 
   \big(\sum_{k=0}^\infty \frac{(t \mathbf{\Lambda})^k}{k!}\big)
   \mathbf{T}^{-1},
\end{equation}
\end{small}
where $\mathbf{\Lambda}$ is the diagonal matrix whose non-zero entries contain the eigenvalues of $\mathbf{J}$ (\ie $\lambda_i$), and $\mathbf{T}$ has the eigenvectors of $\mathbf{J}$ as columns. We observe that the qualitative behavior of $\partial \mathbf{x}(t)/\partial \mathbf{x}(0)$ is determined by the real parts of the eigenvalues of $\mathbf{J}$, leading to three different behaviors: (i) instability, (ii) dissipativity (i.e., information loss), (iii) non-dissipativity (i.e., information preservation). 

\textbf{Instability} is observed when $\max_{i=1,...,d} Re(\lambda_i(\mathbf{J})) > 0$, i.e., a small perturbation of the initial condition would cause an exponential divergence in node representations. 

\textbf{Dissipativity} occurs when $\max_{i=1,...,d} Re(\lambda_i(\mathbf{J})) < 0$. Only local neighborhood information is preserved by the system since the term $\partial \mathbf{x}(t)/\partial \mathbf{x}(0)$ would vanish exponentially fast over time, thereby making the nodes' representation insensitive to differences in the input graph.

\textbf{Non-dissipativity} is obtained when $Re(\lambda_i(\mathbf{J}))= 0$. Here, the system is stable, and the magnitude of $\partial \mathbf{x}(t)/\partial \mathbf{x}(0)$ is constant over time, and the input graph information is effectively propagated through the successive transformations into the final nodes' representations, addressing oversquashing.

\section{SWAN: Space-Weight Antisymmetric GNN}
\label{sec:method}
We now turn to present \textbf{SWAN}, \textbf{s}pace and \textbf{w}eight  \textbf{an}tisymmetric GNN. We analyze its theoretical behavior and show that it is both \emph{global} (\ie graph-wise) and \emph{local} (\ie node-wise) \emph{non-dissipative}. As a consequence, one of the key features of SWAN is that it has a constant \emph{global} information flow rate, unlike common diffusion GNNs. Therefore, SWAN should theoretically be able to propagate information between any nodes with a viable path in the graph, allowing to mitigate oversquashing.  
Figure~\ref{fig:multiple_behaviors}  exemplifies the differences between dissipative, local non-dissipative, and global and local non-dissipative systems. 

\begin{figure*}[h]
\centering
\begin{subfigure}{0.33\textwidth}
    \includegraphics[width=\textwidth]{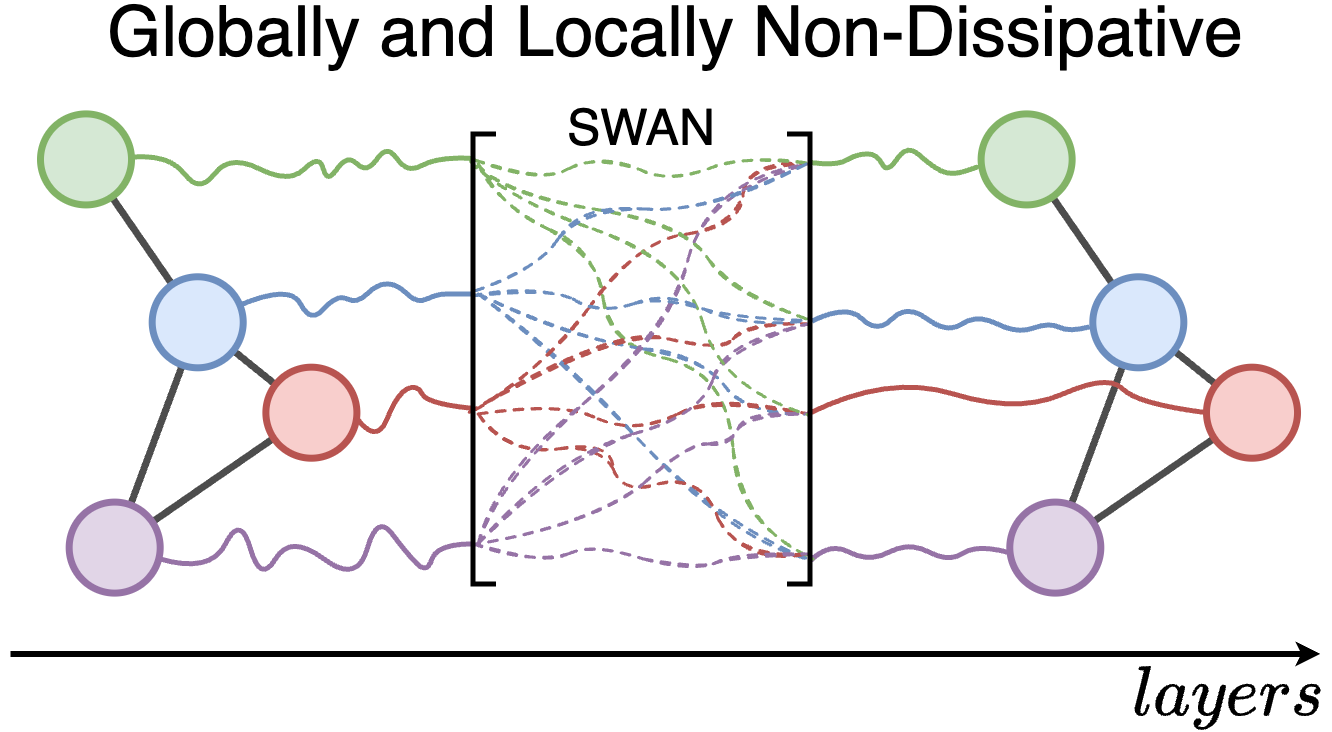}
    \caption{}
        \label{fig:globalNonDis}
\end{subfigure}
\begin{subfigure}{0.33\textwidth}
    \includegraphics[width=\textwidth]{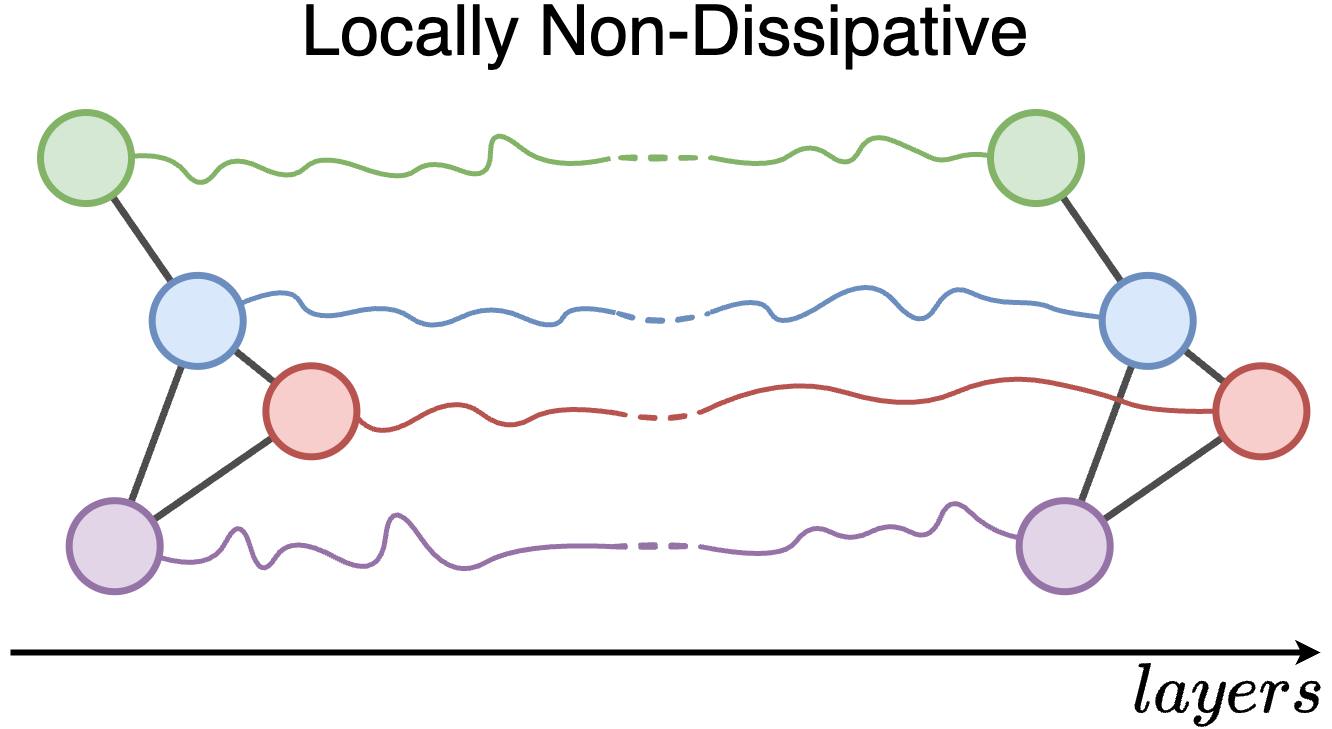}
    \caption{}
    \label{fig:localNonDis}
\end{subfigure}
\begin{subfigure}{0.33\textwidth}
    \includegraphics[width=\textwidth]{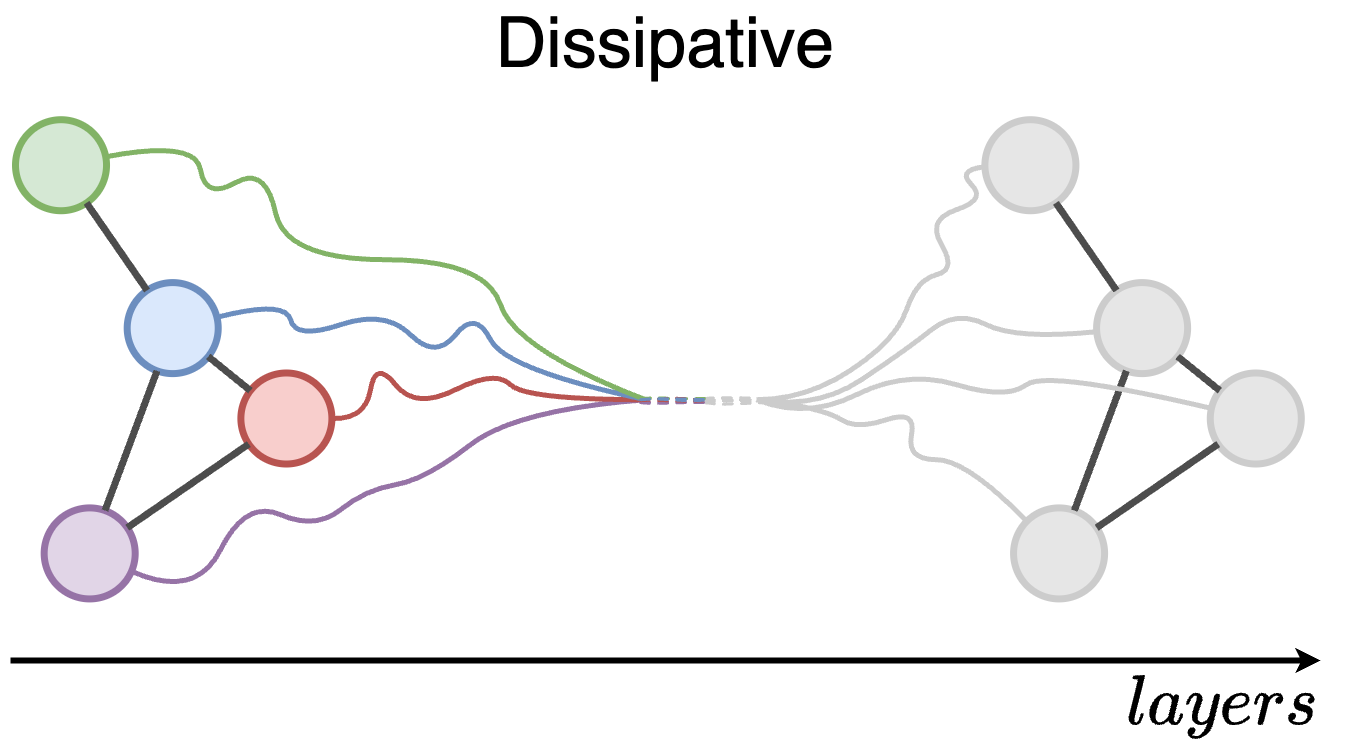}
    \caption{}
        \label{fig:dissipative}

\end{subfigure}
    \caption{The difference between non-dissipative and dissipative behaviors. With global (\ie graph-wise) and local (\ie node-wise) 
    non-dissipative behavior (a), information is propagated between any pair of nodes with a viable path in the graph. Therefore, such a behavior increases the long-range effectiveness of the model. A model exhibiting local 
    non-dissipative behavior (b) enhances only the long-term memory capacity of individual nodes. A model demonstrating dissipative behavior (c) exhibits a convergence of node features toward non-informative values.}
    \label{fig:multiple_behaviors}
\end{figure*}

\noindent \textbf{Space and Weight Antisymmetry.} We define SWAN by including a new term that introduces antisymmetry in the aggregation function, i.e.,
\begin{small}
 \begin{align} \label{eq:new_adgn_nodewise}
     \frac{\partial\mathbf{x}_u(t)}{\partial t} =  \sigma\Bigl(&(\mathbf{W}-\mathbf{W}^\top)\mathbf{x}_u(t) \nonumber
     + \Phi(\{\mathbf{x}_j\}_{j\in\mathcal{N}_u}, \mathbf{V}) \nonumber\\  
     &+ \beta\Psi(\{\mathbf{x}_j\}_{j\in\mathcal{N}_u}, \mathbf{Z})\Bigr).
 \end{align}
\end{small}
where $\mathbf{W}$, $\mathbf{V}$, and $\mathbf{Z}$ are learnable weight matrices. $\Phi$ and $\Psi$ are permutation invariant neighborhood aggregation functions, where $\Psi$ performs antisymmetric aggregation. 
While $\Psi$ can assume various forms of  antisymmetric aggregation functions with imaginary eigenvalues, and $\Phi$ can be any aggregation function, in this paper we explore the following family of parameterizations:
\begin{small}
\begin{align}
\Phi &= \sum_{v\in\mathcal{N}_u}(\hat{\bfA}_{uv}+\hat{\bfA}_{vu})(\mathbf{V}-\bfV^\top)\mathbf{x}_v(t), \label{eq:aggFunc1}\\
\Psi &= \sum_{v\in\mathcal{N}_u}(\tilde{\bfA}_{uv}-\tilde{\bfA}_{vu})(\mathbf{Z}+\mathbf{Z}^\top)\mathbf{x}_v(t), \label{eq:aggFunc2}
\end{align}    
\end{small}
where $\tilde{\mathbf{A}},\hat{\mathbf{A}} \in \mathbb{R}^{n \times n}$ are neighborhood aggregation matrices that can be pre-defined or learned.  In our experiments, we consider two instances of \cref{eq:aggFunc1,eq:aggFunc2}. The first, where $\tilde{\mathbf{A}},\hat{\mathbf{A}}$ are pre-defined by the random walk and symmetrically normalized adjacency matrices, respectively. The second, learns $\tilde{\mathbf{A}},\hat{\mathbf{A}}$, as described in Appendix A. Notably, as shown below, the 
parametrization of $\Psi$ and $\Phi$, described in \cref{eq:aggFunc1,eq:aggFunc2}, 
ensures that SWAN’s node- and graph-wise Jacobians have purely imaginary eigenvalues (\cref{sec:nodewise,sec:graphwise}), which, together with the background in \cref{sec:math}, shows SWAN's ability to be non-dissipative both locally and globally, 
leading to a \emph{globally} constant information flow 
regardless of time, 
\ie the model's depth. 
Lastly, we note that the general formulation of $\Phi$ and $\Psi$ provide a general design principle for introducing non-dissipativity as an inductive bias in any DE-GNN (see Appendix A.4).

\subsection{Node-wise Analysis of SWAN} 
\label{sec:nodewise}
We reformulate 
\cref{eq:new_adgn_nodewise} to consider 
the formulation of $\Phi$ and $\Psi$ as in \cref{eq:aggFunc1,eq:aggFunc2}, reading:
\begin{small}
 \begin{align}\label{eq:new_adgn_node} \nonumber
     \frac{\partial\mathbf{x}_u(t)}{\partial t} =  \sigma&\Bigl((\mathbf{W}-\mathbf{W}^\top)\mathbf{x}_u(t) \nonumber\\ 
     &+ \sum_{v\in\mathcal{N}_u}(\hat{\bfA}_{uv}+\hat{\bfA}_{vu})(\mathbf{V}-\bfV^\top)\mathbf{x}_v(t) \nonumber\\  
     &+ \beta\sum_{v\in\mathcal{N}_u}((\tilde{\bfA}_{uv}-\tilde{\bfA}_{vu})(\mathbf{Z}+\mathbf{Z}^\top)\mathbf{x}_v(t)\Bigr).
 \end{align}
\end{small}
\textbf{SWAN is locally non-dissipative.} Using the sensitivity analysis in Section~\ref{sec:math}, we show that SWAN is stable and non-dissipative from a node perspective, \ie it is \emph{locally} non-dissipative. 
In this case, the Jacobian $\mathbf{J}(t) = \mathbf{M}_1\mathbf{M}_2$ of \cref{eq:new_adgn_node} is composed of :
\begin{small}
 \begin{align}
 \nonumber \mathbf{M}_1 &= \mathrm{diag}\Bigl[\sigma'\Bigl((\mathbf{W}-\mathbf{W}^\top)\mathbf{x}_u(t) \\ \nonumber &\hspace{1cm}+ \sum_{v\in\mathcal{N}_u}(\hat{\bfA}_{uv}+\hat{\bfA}_{vu})(\mathbf{V}-\bfV^\top)\mathbf{x}_v(t) \\  &\hspace{1cm}+ \beta\sum_{v\in\mathcal{N}_u}(\tilde{\bfA}_{uv}-\tilde{\bfA}_{vu})(\mathbf{Z}+\mathbf{Z}^\top)\mathbf{x}_v(t)\Bigl)\Bigr],\\
 \mathbf{M}_2 &= (\mathbf{W}-\mathbf{W}^\top) + (\hat{\bfA}_{uv}+\hat{\bfA}_{vu})(\mathbf{V}-\bfV^\top).
 \end{align}
\end{small}
Following results from \cite{chang2018antisymmetricrnn,gravina2022anti}, only the eigenvalues of $\mathbf{M}_2$ determine the \emph{local} 
stability and non-dissipativity of the system in \cref{eq:new_adgn_nodewise} for the final behavior of the model. Specifically, if the real part of all the eigenvalues of $\bfM_2$ is zero, then stability and non-dissipativity are achieved. We note that this is indeed the case in our system, since the real part of the eigenvalues of antisymmetric matrices is zero, and $\bfM_2$ is composed of a summation of two antisymmetric matrices.

\subsection{Graph-wise Analysis of SWAN}
\label{sec:graphwise}
While the node-perspective analysis is important because it shows the long-term memory capacity of individual nodes, as illustrated in Figure~\ref{fig:localNonDis}, it overlooks  \emph{pairwise} node interactions, which are described by the properties of \cref{eq:new_adgn_nodewise} with respect to the graph, illustrated in Figure \ref{fig:globalNonDis}.
As we show below, our SWAN is globally non-dissipative. Hence, it allows the constant rate of information flow and node interactions, independently of time $t$, i.e., the network's depth.  Therefore, we deem that SWAN's non-dissipativity behavior is beneficial in addressing oversquashing in MPNNs.

\noindent  \textbf{SWAN is globally non-dissipative.} We start by reformulating \cref{eq:new_adgn_node} from a node-wise formulation to a graph-perspective formulation, as follows:
\begin{small}
\begin{align}\label{eq:new_adgn_graphwise}
    \nonumber \frac{\partial\mathbf{X}(t)}{\partial t} = \sigma \Bigl(\mathbf{X}(t)&{(\mathbf{W}-\mathbf{W}^\top)} 
    +{(\hat{\mathbf{A}}+\hat{\mathbf{A}}^\top)\mathbf{X}(t)(\mathbf{V}-\bfV^\top)} \\  &
    + \beta{(\tilde{\mathbf{{A}}}-\tilde{\mathbf{{A}}}^\top)\mathbf{X}(t)(\mathbf{Z}+\mathbf{Z}^\top)}\Bigr).
\end{align}
\end{small}
%
Following the sensitivity analysis introduced in Section~\ref{sec:math}, and applying
the vectorization operator, 
the Jacobian of \cref{eq:new_adgn_graphwise}, $\mathbf{J}(t) = \mathbf{M}_1\mathbf{M}_2$, writes as the multiplication of:
\begin{small}
 \begin{align}
 \label{eq:J_SWAN} \nonumber
 \nonumber \mathbf{M}_1 &= \mathrm{diag}\Bigl[\mathrm{vec}\Bigl(\sigma'\Bigl(\mathbf{I}\mathbf{X} (\mathbf{W}-\mathbf{W}^\top) \\ \nonumber&\hspace{1cm}+ (\hat{\mathbf{A}}+\hat{\mathbf{A}}^\top)\mathbf{X}(t)(\mathbf{V}-\bfV^\top) \\ &\hspace{1cm}+\beta(\tilde{\mathbf{{A}}}-\tilde{\mathbf{{A}}}^\top)\mathbf{X}(t)(\mathbf{Z}+\mathbf{Z}^\top) \Bigr)\Bigr) \Bigr] \\
 \label{eq:graphwiseM2}
 \nonumber \mathbf{M}_2 &= (\mathbf{W}-\mathbf{W}^\top)^\top \otimes \bfI \\ \nonumber &\hspace{1cm} + (\mathbf{V}-\bfV^\top)^\top\otimes (\hat{\mathbf{A}}+\hat{\mathbf{A}}^\top) \\  &\hspace{1cm} + \beta(\mathbf{Z}+\mathbf{Z}^\top)^\top \otimes (\tilde{\mathbf{{A}}}-\tilde{\mathbf{{A}}}^\top),
     \end{align}
\end{small}
where $\bfI \in \mathbb{R}^{n \times n}$ is the identity matrix, $\rm{vec}$ is the vectorization operator, and $\otimes$ is the Kronecker product (see Appendix H for more details).
As in our node-wise analysis in Section \ref{sec:nodewise}, $\mathbf{M}_1$ in \cref{eq:J_SWAN} is a diagonal matrix. Thus, stability and non-dissipativity demand that $\bfM_2$ from \cref{eq:graphwiseM2} has eigenvalues with real part equal to zero. We see that $\bfM_2$ satisfies this condition as it is composed of a summation of three antisymmetric matrices whose eigenvalues have a real part of zero. We conclude that SWAN (\cref{eq:new_adgn_graphwise}) 
is stable, and globally and locally non-dissipative. 
We now show that the properties of 
stability and global non-dissipativity allow the design of GNNs that can mitigate oversquashing.

\begin{theorem}[SWAN has a constant global  information propagation rate]
\label{thm:swanConstantRate} The information propagation rate among the graph nodes $\cal{V}$ is constant, $c$, independently of time $t$:
\begin{small}
\begin{equation}\label{eq:swanConstant}   
\left\Vert\frac{\partial \rm{vec}(\mathbf{X}(t))}{\partial \rm{vec}(\mathbf{X}(0))}\right\Vert = c,
\end{equation}
\end{small}
\end{theorem}

\begin{proof}
Let us consider the following equation:
\begin{small}
\begin{equation}    \label{eq:proof_graph_wise}   
\frac{d}{dt}\left(\frac{\partial \mathbf{X}(t)}{\partial \mathbf{X}(0)}\right) = \frac{d}{dt}\left(\frac{\partial \rm{vec}(\mathbf{X}(t))}{\partial \rm{vec}(\mathbf{X}(0))}\right) = \mathbf{J}(t) \frac{\partial \rm{vec}(\mathbf{X}(t))}{\partial\rm{vec}(\mathbf{X}(0))}. 
\end{equation}
\end{small}
We follow the assumption in \cite{chang2018antisymmetricrnn,gravina2022anti} that the Jacobian, $\mathbf{J}(t)$, does not change significantly over time, 
then we can apply results from autonomous differential equations and solve \cref{eq:proof_graph_wise}: 
\begin{small}
    \begin{equation}
   \frac{\partial \rm{vec}(\mathbf{X}(t))}{\partial\rm{vec}(\mathbf{X}(0))} = e^{t \mathbf{J}} = \mathbf{T} e^{t \mathbf{\Lambda}}\mathbf{T}^{-1} = \mathbf{T} 
   \big(\sum_{k=0}^\infty \frac{(t \mathbf{\Lambda})^k}{k!}\big)
   \mathbf{T}^{-1},
\end{equation}
\end{small}
where $\mathbf{\Lambda}$ is the diagonal matrix whose non-zero entries contain the eigenvalues of $\mathbf{J}$, and $\mathbf{T}$ has the eigenvectors of $\mathbf{J}$ as columns, as in Section~\ref{sec:math}. 
As previously shown, 
it holds that 
 $Re(\lambda_i(\mathbf{J}(t))) 
= 0$ for $i=1,...,d$, since the Jacobian is the result of the multiplication between a diagonal matrix and an antisymmetric matrix. Thus, the magnitude of $\partial \mathbf{X}(t)/\partial \mathbf{X}(0)$ is constant over time, allowing input features to propagate through the layers into the final node features. 
\end{proof}

\Cref{thm:swanConstantRate} states that regardless of time $t$ (equivalent to $\ell = t/\epsilon$ layers of SWAN, where $\epsilon$ is the step size), the information between nodes continues to propagate at a constant rate, unlike diffusion  GNNs that exhibit an exponential decay in the propagation rate with respect to time, as shown below.

\begin{theorem}[Time Decaying Propagation in Diffusion GNNs] \label{thm:diffusionExpDecay} A diffusion GNN with Jacobian eigenvalues with magnitude $\bfK_{ii}= |\mathbf{\Lambda}_{ii}| \ , \ i\in\{0,\ldots,n-1\}$ has an exponentially decaying information propagation rate, as follows:
\begin{small}
\begin{equation}    
    \label{eq:diffusionDecay}   
    \left\Vert\frac{\partial \rm{vec}(\mathbf{X}(t))}{\partial \rm{vec}(\mathbf{X}(0))}\right\Vert = \|e^{-t\mathbf{K}}\|, \end{equation}
\end{small}
\end{theorem}

See proof in Appendix B.
This result means that for diffusion methods, as $t$ grows, deeper layers are not able to share new information between nodes as effectively as earlier layers of the network. On the contrary, our SWAN maintains the same effectiveness, independently of time, meaning it retains its ability to share information across nodes with the same effectiveness in each layer of the network, regardless of its depth, as illustrated in Figure~\ref{fig:propagation} and Figure~\ref{fig:multiple_behaviors}.

\subsection{The Benefit of Spatial Antisymmetry}
\label{sec:benefitSpatial}
While oversquashig was not mathematically defined in \cite{alon2021oversquashing}, it was recently proposed in \cite{topping2022understanding, diGiovanniOversquashing} to quantify the level, or lack of oversquashing, by measuring the sensitivity of node embedding after $\ell$ layers with respect to the input of another node 
$\left\|\partial\mathbf{x}_v(\ell)/\partial\mathbf{x}_u(0)\right\|$, which can be bounded as follows: 
\begin{small}
\begin{equation}\label{eq:mpnn_oversquashing_1}
    \left\|\frac{\partial\mathbf{x}_v(\ell)}{\partial\mathbf{x}_u(0)}\right\| \leq \underbrace{(c_\sigma w p)^\ell}_{model}\underbrace{(\mathbf{O}^\ell)_{vu}}_{topology},
\end{equation}    
\end{small}
where $c_\sigma$ is the Lipschitz constant of the activation $\sigma$, $w$ is the maximal entry-value over all weight matrices, and $p$ is the embedding dimension. The term $\mathbf{O}=c_r\mathbf{I}+c_a\mathbf{A}\in\mathbb{R}^{n\times n}$ is the message passing matrix adopted by the MPNN, with $c_r$ and $c_a$ are the contributions of the residual and aggregation term. Oversquashing occurs if the right-hand side of \cref{eq:mpnn_oversquashing_1} is too small \citep{diGiovanniOversquashing}.

\noindent \textbf{The sensitivity of SWAN.}
We present the sensitivity bound of SWAN, with its proof in Appendix C. 
\begin{theorem}[SWAN sensitivity upper bound]\label{thm:swan_sensitivity}
    Consider SWAN (\cref{eq:new_adgn_nodewise}), with $\ell$ layers, and $u,v\in\mathcal{V}$ two connected nodes of the graph. The sensitivity of $v$'s embedding after $\ell$ layers with respect to the input of node $u$ is 
    \begin{small}
    \begin{equation}\label{eq:swan_oversquashing}
    \left\|\frac{\partial\mathbf{x}_v(\ell)}{\partial\mathbf{x}_u(0)}\right\| \leq \underbrace{(c_\sigma w p)^\ell}_{model}\underbrace{((c_r \mathbf{I} + c_a \mathbf{A} + \beta c_b \mathbf{S})^\ell)_{vu}}_{topology}
    \end{equation}
    \end{small}
     with $c_\sigma$ the Lipschitz constant of non-linearity $\sigma$, $w$ is the maximal entry-value of all weight matrices, $p$ the embedding dimension, $\mathbf{A}$ the graph shift operator, {$\mathbf{S}=(\tilde{\mathbf{{A}}}-\tilde{\mathbf{{A}}}^\top)$} the antisymmetric graph operator, and $c_r$ and $c_a$ the weighted contribution of the residual term and aggregation term.
\end{theorem}

The result of \cref{thm:swan_sensitivity} indicates that the added antisymmetric term $\Psi$ contributes to an increase in the measured upper bound. This result, together with the constant rate of information flow obtained from \cref{thm:swanConstantRate}, holds the potential to theoretically mitigate oversquashing using SWAN.

\section{Experiments}\label{sec:experiments}

\textbf{Objectives.} We evaluate our SWAN and compare it with various methods.
Specifically, we seek to address the following questions:
\begin{enumerate*}[label=(\roman*)]
    \item Can SWAN effectively propagate information to distant nodes?
    \item Can SWAN accurately predict graph properties related to long-range interactions?
    \item How does SWAN perform on real-world long-range benchmarks?
\end{enumerate*}

\noindent  \textbf{Baselines.} We consider \begin{enumerate*}[label=(\roman*)]
\item MPNNs with linear complexity (similar to the complexity of our SWAN), \ie GCN~\cite{kipf2016semi}, GraphSAGE~\cite{hamilton2017inductive}, GAT~\cite{velickovic2018graph}, GatedGCN~\cite{gatedgcn}, GIN~\cite{xu2019how}, GINE~\cite{Hu2020Strategies}, and GCNII~\cite{chen20simple}.
\item DE-GNNs such as DGC~\cite{DGC}, GRAND~\cite{chamberlain2021grand}, GraphCON~\cite{rusch2022graph}, and ADGN~\cite{gravina2022anti}.
\item Oversquashing designated methods such as Graph Transformers and Higher-Order GNNs, \ie Transformer~\cite{vaswani2017attention, dwivedi2021generalization}, DIGL~\cite{gasteiger_diffusion_2019}, MixHop~\cite{abu2019mixhop}, SAN~\cite{kreuzer2021rethinking}, GraphGPS~\cite{rampasek2022GPS}, and DRew~\cite{drew}.
\end{enumerate*}

\noindent \textbf{Experimental Details.} We discuss the datasets in our experiments in Appendix D, and the complexity of SWAN and provide runtimes in Appendix F. 
Our implementation uses PyTorch, and is available at  \url{https://github.com/gravins/SWAN}.

\noindent  \textbf{SWAN Variants.}  In the following experiments, we leverage the general formulation of our method, discretized by forward Euler (see Appendix A.1), and explore two main variants of SWAN, each distinguished by the implementation of the aggregation terms $\hat{\bfA}, \ \tilde{\bfA}$ in the functions $\Psi$ and $\Phi$, as shown in \cref{eq:aggFunc1,eq:aggFunc2}. Specifically, we consider (i) SWAN, which implements the aggregation terms using pre-defined operators, which are the symmetric normalized and random walk adjacency matrices, as described in Appendix A.2, and, (ii) SWAN-\textsc{learn} which utilizes the learned aggregation terms described in Appendix A.2.
Both variants follow the form of \cref{eq:aggFunc1,eq:aggFunc2}, in line with the theoretical analysis in \cref{sec:method}, and in particular \cref{thm:swanConstantRate}.
We refer the reader to Appendix E.1 for a detailed description. 

\subsection{Graph Transfer}\label{sec:exp_transfer}
\textbf{Setup.}
We consider a graph transfer task, where the goal is to transfer a label from a source to a target node, with a distance of $k$ hops. We note that this task can be effectively solved only by non-dissipative methods that preserve source information. 
We initialize nodes with a random valued feature 
and we assign values ``1'' and ``0'' to source and target nodes, respectively. We consider three graph distributions, \ie line, ring, crossed-ring, as illustrated in Figure 4, with four different distances $k=\{3,5,10,50\}$. Increasing $k$ increases the task complexity. In Appendix D.1 and E.2, we provide additional details about the dataset and the task.

\noindent  \textbf{Results.}
Figure~\ref{fig:graph_transfer_res} reports the results on graph transfer tasks. Overall, baseline methods struggle to accurately transfer the information through the graph, especially when the distance is high, \ie $\#hops\geq10$. Differently, non-dissipative methods, such as ADGN and SWAN, achieve low errors across all distances. Moreover, SWAN consistently outperforms ADGN, empirically supporting our theoretical findings that SWAN can better propagate information among distant nodes.
\begin{figure*}[t]
\centering
\includegraphics[width=0.8\linewidth]{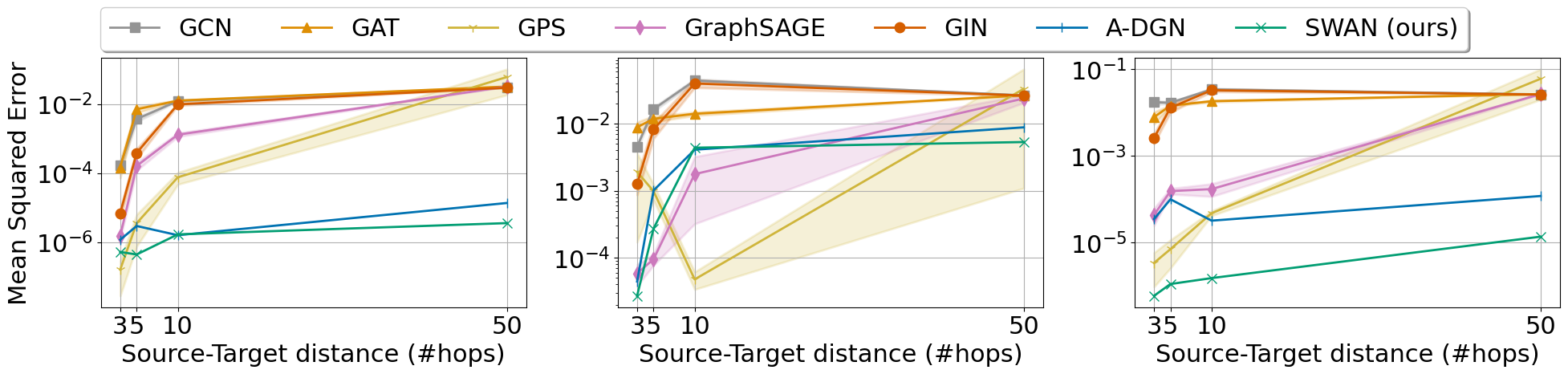}\\
{\footnotesize \hspace{15mm}(a) Line \hspace{45mm}(b) Ring \hspace{35mm}(c) Crossed-Ring}
\caption{Information transfer performance on (a) Line, (b) Ring, and (c) Crossed-Ring graphs. Non-dissipative methods like ADGN and SWAN allow for the accurate transfer of information.}
\label{fig:graph_transfer_res}
\end{figure*}
\subsection{Graph Property Prediction}\label{sec:exp_graph_prop_pred}
\begin{table}[t]
\setlength{\tabcolsep}{1pt}
\centering
\footnotesize
\vspace{-2mm}
\begin{tabular}{lccc}
\hline\toprule
\textbf{Model} &\textbf{Diameter} & \textbf{SSSP} & \textbf{Eccentricity} \\\midrule
\textbf{MPNNs} \\
$\,$ GCN            & 0.7424$_{\pm0.0466}$ & 0.9499$_{\pm9.18\cdot10^{-5}}$ & 0.8468$_{\pm0.0028}$ \\
$\,$ GAT            & 0.8221$_{\pm0.0752}$ & 0.6951$_{\pm0.1499}$           & 0.7909$_{\pm0.0222}$  \\
$\,$ GraphSAGE      & 0.8645$_{\pm0.0401}$ & 0.2863$_{\pm0.1843}$           &  0.7863$_{\pm0.0207}$\\
$\,$ GIN            & 0.6131$_{\pm0.0990}$ & -0.5408$_{\pm0.4193}$          & 0.9504$_{\pm0.0007}$\\
$\,$  GCNII          & 0.5287$_{\pm0.0570}$ & -1.1329$_{\pm0.0135}$          & 0.7640$_{\pm0.0355}$\\
\midrule
\textbf{DE-GNNs} \\
$\,$ DGC            & 0.6028$_{\pm0.0050}$ & -0.1483$_{\pm0.0231}$          & 0.8261$_{\pm0.0032}$\\
$\,$ GRAND          & 0.6715$_{\pm0.0490}$ & -0.0942$_{\pm0.3897}$          & \four{0.6602$_{\pm0.1393}$} \\
$\,$ GraphCON       & \four{0.0964$_{\pm0.0620}$} & \four{-1.3836$_{\pm0.0092}$} & 0.6833$_{\pm0.0074}$\\

$\,$ ADGN
& \three{-0.5188$_{\pm0.1812}$} & \two{-3.2417$_{\pm0.0751}$} & \three{0.4296$_{\pm0.1003}$}  \\

\midrule
\textbf{Ours} \\


$\,$ 
SWAN & \two{-0.5249$_{\pm0.0155}$} &  \three{-3.2370$_{\pm0.0834}$} & \two{0.4094$_{\pm0.0764}$} \\


$\,$ 
SWAN-\textsc{learn} & \one{-0.5981$_{\pm0.1145}$}  & \one{-3.5425$_{\pm0.0830}$}  & \one{-0.0739$_{\pm0.2190}$} \\



\bottomrule\hline      
\end{tabular}
\caption{Mean test set {\small$log_{10}(\mathrm{MSE})$} and std averaged over 4 random weight initializations on the Graph Property Prediction tasks. The lower, the better. 
\one{First}, \two{second}, \three{third} and \four{fourth} best results for each task are color-coded. 
\label{tab:results_GraphProp_complete}
}
\end{table}

\textbf{Setup.}
We consider the prediction of three graph properties -  Diameter, Single-Source Shortest Paths (SSSP), and node Eccentricity on synthetic graphs \cite{PNA}, and follow the setup outlined in \citet{gravina2022anti}, further discussed in Appendix D.2 and E.3. These three tasks rely on long-range interactions for shortest-path calculations, as in Bellman-Ford and Dijkstra's algorithms.

\noindent  \textbf{Results.}
In \cref{tab:results_GraphProp_complete}, we compare SWAN and SWAN-\textsc{learn} with  MPNNs and DE-GNNs, indicating that SWAN consistently improves performance, where SWAN-\textsc{learn} yields the best results, with an improvement of up to 117\% of the runner-up model. In Appendix G we also provide results for SWAN with layer-dependent weights showing improved performance.
We note that similarly to the transfer task in Section~\ref{sec:exp_transfer}, solving  Graph Property Prediction tasks necessitates capturing long-term dependencies. Hence, successful prediction requires the mitigation of oversquashing. For instance, in the eccentricity task, the goal is to calculate the maximal shortest path between node $u$ and all other nodes -- a task requiring propagation from distant nodes. 

\subsection{Long-Range Graph Benchmark}\label{sec:exp_lrb}
\noindent  \textbf{Setup.} We follow the settings from 
\citet{dwivedi2022LRGB}  on the Peptides-func, Peptides-struct, and PascalVOC-sp datasets, on which we elaborate in Appendix D.3 and E.4. As metrics, we consider the average precision (AP) for the Peptides-func, mean-absolute-error (MAE) for Peptides-struct, and the macro-weighted F1 score for PascalVOC-sp. 

\noindent  \textbf{Results.}
Table \ref{tab:lrgb_results} compares SWAN with various MPNNs, DE-GNNs, and Graph Transformers. 
\begin{table}[t]
\footnotesize
\setlength{\tabcolsep}{3.5pt}
\centering
\footnotesize
\vspace{-2mm}
\begin{tabular}{@{}lccc@{}}
\hline\toprule
\multirow{3}{*}{\textbf{Model}} & \textbf{Peptides-}  & \textbf{Peptides-} & \textbf{Pascal}                \\
& \textbf{func} & \textbf{struct} & \textbf{VOC-sp}
                               \\
                                & \scriptsize{AP $\uparrow$}                             & \scriptsize{MAE $\downarrow$}  & \scriptsize{F1 $\uparrow$}                           \\ \midrule
\textbf{MPNNs} \\
$\,$ GCN         & 59.30$_{\pm0.23}$ & 0.3496$_{\pm0.0013}$ & 12.68$_{\pm0.60}$\\
$\,$ GINE        & 54.98$_{\pm0.79}$ & 0.3547$_{\pm0.0045}$ & 12.65$_{\pm0.76}$\\
$\,$ GCNII       & 55.43$_{\pm0.78}$ & 0.3471$_{\pm0.0010}$ & 16.98$_{\pm0.80}$\\
$\,$ GatedGCN    & 58.64$_{\pm0.77}$ & 0.3420$_{\pm0.0013}$ & 28.73$_{\pm2.19}$\\
 \midrule
\textbf{Transformers} \\
$\,$ Transformer+LapPE  & 63.26$_{\pm1.26}$          & \three{0.2529$_{\pm0.0016}$}        & 26.94$_{\pm0.98}$ \\
$\,$ SAN+LapPE          & \three{63.84$_{\pm1.21}$}  & 0.2683$_{\pm0.0043}$                & \two{32.30$_{\pm0.39}$} \\
$\,$ GraphGPS+LapPE     & \two{65.35$_{\pm0.41}$}    & \two{0.2500$_{\pm0.0005}$} & \one{37.48$_{\pm1.09}$}\ \\ \midrule

\textbf{DE-GNNs} \\
$\,$ GRAND    & 57.89$_{\pm0.62}$ & 0.3418$_{\pm0.0015}$ &  19.18$_{\pm0.97}$  \\
$\,$ GraphCON & 60.22$_{\pm0.68}$ & 0.2778$_{\pm0.0018}$ &  21.08$_{\pm0.91}$ \\
$\,$ ADGN     & 59.75$_{\pm0.44}$ & 0.2874$_{\pm0.0021}$ &  23.49$_{\pm0.54}$ \\ 
\midrule
\textbf{Ours} \\    
$\,$ SWAN                & 63.13$_{\pm0.46}$ & 0.2571$_{\pm0.0018}$              & 27.96$_{\pm0.48}$\\
$\,$ SWAN-\textsc{learn} & \one{67.51$_{\pm0.39}$} & \one{0.2485$_{\pm0.0009}$}  & \three{31.92$_{\pm2.50}$} \\
\bottomrule\hline
\end{tabular}
\caption{Performance of standard MPNNs, graph Transformers, DE-GNNs, and our SWAN across three LRGB tasks. Results are averaged over 3 weight initializations. The \one{first}, \two{second}, and \three{third} best results for each task are color-coded.
\label{tab:lrgb_results}
}
\end{table}
In Appendix G, we also provide a comparison with multi-hop methods. 
Our results in Table \ref{tab:lrgb_results} suggest the following:
\begin{enumerate*}[label=(\roman*)]
    \item SWAN achieves significantly better results than standard MPNNs such as GCN, GINE, or GCNII. For example, on Peptides-func, SWAN-\textsc{learn} achieves an average precision of 66.54, while GCN achieves a score of 59.30.
    \item Compared with Transformers, which are of complexity ${\cal{O}}(|V|^2)$, our SWAN achieves better performance while remaining with a linear complexity of $\cal{O}(|V|+|E|)$.
    \item Among its class of DE-GNNs, our SWAN offers overall better performance.
\end{enumerate*}

\begin{table*}[t]
\centering
\footnotesize
\setlength{\tabcolsep}{3.5pt}
\begin{tabular}{lcccccc}
\hline\toprule
\footnotesize
\textbf{Dataset $\downarrow$ / Model $\rightarrow$} & \textbf{SWAN$_{\beta=0}$} & \textbf{SWAN-\textsc{ne}} & \textbf{SWAN-\textsc{ne-learn}} & \textbf{SWAN} & \textbf{SWAN-\textsc{learn}} \\\midrule
\textbf{Diam.} \scriptsize{($log_{10}$(MSE)$\downarrow$)} & -0.3882$_{\pm0.0610}$ & \three{-0.5497$_{\pm0.0766}$} & \two{-0.5631$_{\pm0.0694}$} & -0.5249$_{\pm0.0155}$ & \one{-0.5981$_{\pm0.1145}$} \\
\textbf{SSSP} \scriptsize{($log_{10}$(MSE)$\downarrow$)} & -3.2061$_{\pm0.0416}$ & -3.1913$_{\pm0.0762}$ & \two{-3.5296$_{\pm0.0831}$} & \three{-3.2370$_{\pm0.0834}$} & \one{-3.5425$_{\pm0.0830}$} \\
\textbf{Ecc.} \scriptsize{($log_{10}$(MSE)$\downarrow$)} & 0.5573$_{\pm0.0247}$ & \three{0.3792$_{\pm0.1514}$} & \two{0.1317$_{\pm0.1253}$} & 0.4094$_{\pm0.0764}$ & \one{-0.0739$_{\pm0.219}$} \\
\textbf{Peptides-func} \scriptsize{(AP $\uparrow$)} & 61.95$_{\pm00.67}$ & 61.19$_{\pm0.37}$ & \three{62.49$_{\pm0.51}$} & \two{63.13$_{\pm0.46}$} & \one{67.51$_{\pm0.39}$} \\
\textbf{Peptides-struct} \scriptsize{(MAE $\downarrow$)} & 0.2703$_{\pm0.0023}$ & 0.2672$_{\pm0.0012}$ & \three{0.2606$_{\pm0.0007}$} & \two{0.2571$_{\pm0.0018}$} & \one{0.2485$_{\pm0.0009}$} \\
\bottomrule\hline
\end{tabular}
\caption{Performance of different versions of SWAN on Graph Property Prediction and LRGB tasks. Results are averaged over 4 random weight initializations on the Graph Property Prediction, while over 3 on the LRGB. The \one{first}, \two{second}, and \three{third} best results for each task are color-coded. Global and Local Non-Dissipative variants achieve highest performance.\label{tab:importance}}
\end{table*}

\subsection{Ablation Study}
\label{sec:ablation}
We now study the contribution of global (graph-wise) and local (node-wise) non-dissipativity, and spatial antisymmetry.

\noindent  \textbf{The importance of global and local non-dissipativity.}
To verify the contribution of the global and local non-dissipativity in SWAN, we evaluate the performance of several variants that can deviate from being globally and locally non-dissipative, although in a bounded manner, as discussed in Appendix E.1. Specifically, we consider the non-enforced (NE) variants of SWAN and SWAN-\textsc{learn}, which use an unconstrained weight matrix $\mathbf{V}$, rather than forcing it to be antisymmetric as in SWAN. These two additional variants are called SWAN-\textsc{ne} and SWAN-\textsc{learn-ne}, respectively.  Table~\ref{tab:importance} shows the performance of the SWAN on the Property Prediction and LRGB tasks. The highest performance on synthetic and real-world problems was achieved with SWAN-\textsc{learn}, which is globally and locally non-dissipative.

\noindent  \textbf{The benefit of spatial antisymmetry.} 
Table~\ref{tab:importance} highlights the advantages of the spatial antisymmetry in \cref{eq:new_adgn_nodewise}.  By setting $\beta=0$, we obtain a model with antisymmetry solely in the weight space, exhibiting only a local non-dissipative behavior. Our results indicate a noteworthy performance improvement when spatial antisymmetry is employed. This is further supported by the improved performance of -\textsc{ne} versions of SWAN, which do not guarantee both global and local non-dissipative behavior, compared to SWAN$_{\beta=0}$. 

\section{Related Work}
\label{sec:related}
\textbf{Graph Neural Networks based on Differential Equations.}
Adopting the interpretation of convolutional neural networks (CNNs) as discretization of ODEs and PDEs  \citep{RuthottoHaber2018, chen2018neural} to GNNs, works like CGNN \citep{pmlr-v119-xhonneux20a}, GCDE \citep{poli}, GODE \citep{zhuang2020ordinary}, GRAND \citep{chamberlain2021grand}, PDE-GCN\textsubscript{D} \citep{eliasof2021pde},  DGC~\cite{DGC}, and others, propose interpreting GNN layers as discretization steps of the heat equation
. This 
allows controlling the diffusion (smoothing) in the network and understanding the problem of oversmoothing \citep{nt2019revisiting,oono2020graph,cai2020note} in GNNs. 
Differently, \citet{pmlr-v162-choromanski22a} 
propose an architecture with an attention mechanism based on the heat diffusion kernel.
Other architectures like PDE-GCN\textsubscript{M} \citep{eliasof2021pde} and GraphCON \citep{rusch2022graph} propose to mix diffusion and oscillatory processes as a feature energy preservation mechanism. Other recent works 
proposed mechanisms such as anti-symmetry \cite{gravina2022anti}, reaction-diffusion-based dynamics \cite{wang2022acmp, choi2022gread}, 
and advection-reaction-diffusion \cite{eliasof2023adr}. However, most of the aforementioned works are diffusion-based DE-GNNs, limited in modeling long-range interactions, while our SWAN is a DE-GNN with a constant information flow, addressing oversquashing in graphs.
While the aforementioned works focus on the spatial aggregation term of DE-GNNs, the temporal domain of DE-GNNs has also been studied in \cite{eliasof2024temporal,gravina2024ctan, kang2024unleashing,gravina2024temporal}. A  review of these methods is given in \citet{han2023continuous}. 

\noindent  \textbf{Oversquashing in MPNNs.} 
Oversquashing in MPNNs, which hampers information transfer across distant nodes~\cite{alon2021oversquashing}, has prompted various mitigation strategies. \emph{Graph rewiring} methods like SDRF~\cite{topping2022understanding} densify graphs as a preprocessing step, while approaches such as GRAND~\cite{chamberlain2021grand}, BLEND~\cite{blend}, and DRew~\cite{drew} dynamically adjust connectivity based on node features. Transformer-based models~\cite{ dwivedi2021generalization, rampasek2022GPS} bypass oversquashing with all-to-all message passing. Another direction uses \emph{non-local dynamics} to enable dense communication, as in FLODE~\cite{maskey2023fractional}, which leverages fractional graph shifts, QDC~\cite{markovich2023qdc} with quantum diffusion kernels, and G2TN~\cite{g2tn}, which captures diffusion paths. While effective, these methods often increase computational complexity due to dense propagation operators. For further discussion, see \cite{shi2023exposition}. We note that, long-range interactions have also been explored in sequential models~\cite{lstm, ssm}.

\noindent   \textbf{Non-dissipative Systems.} 
 Non-dissipative Systems are characterized by the absence of energy dissipation, thus playing a crucial role in various domains such as physics \cite{goldstein2002classical}, engineering \cite{ogata2010modern}, and machine learning, where such systems were shown to effectively model information flow. For example, the effectiveness of a non-dissipativity was demonstrated to enhance the power of reservoir 
 computing \cite{GALLICCHIO2024127411} and recurrent neural networks \cite{chang2018antisymmetricrnn}. In the context of GNNs, it was shown in \cite{gravina2022anti} that \emph{local} non-dissipativity is beneficial for long interaction modeling.

\section{Summary}
\label{sec:summary}
In this work, we have presented SWAN (Space-Weight ANtisymmetry), a novel Differential Equation GNN (DE-GNN) designed to address the oversquashing problem. SWAN incorporates both global (\ie graph-wise) and local (\ie node-wise) non-dissipative properties through space and weight antisymmetric parameterization, and provides a general design principle for introducing non-dissipativity as an inductive bias in any DE-GNN.
Our theoretical and experimental results emphasize the significance of global and local non-dissipativity, achieved by SWAN. For such reasons, we believe SWAN represents a significant step forward in addressing oversquashing in GNNs.

\section*{Acknowledgments}
The work has been partially supported by EU-EIC EMERGE (Grant No. 101070918).  ME is funded by the Blavatnik-Cambridge fellowship, the Accelerate Programme for Scientific Discovery, and the Maths4DL EPSRC Programme.

\clearpage

\appendix
\onecolumn

\section{Architectural Details}
\label{app:architecture}
\subsection{Integration of SWAN}\label{app:swan_integration}
The ODE that defines SWAN follows the general DE-GNNs form, presented and discussed in \cref{sec:degnns}. While there are various ways to integrate these equations (see for example various integration techniques in \cite{AscherPetzoldODEs}), we follow the common forward Euler discretization approach, which is abundantly used in DE-GNNs literature \cite{chamberlain2021grand, eliasof2021pde, gravina2022anti, rusch2022graph}. Formally, using the forward Euler method to discretize SWAN (\cref{eq:new_adgn_graphwise}) yields the following graph neural layer:
\begin{small}
\begin{equation}
    \label{eq:discretization}
    \bfX^{(\ell+1)} = \bfX^{(\ell)} + \epsilon \sigma \Bigl(\mathbf{X}^{(\ell)}{(\mathbf{W}-\mathbf{W}^\top)} \nonumber\\ 
    +{\Phi(\mathbf{A},\mathbf{X}^{(\ell)},\mathbf{V})} \nonumber\\ 
    + \beta{\Psi(\mathbf{A}, \mathbf{X}^{(\ell)}, \mathbf{Z})}\Bigr).
\end{equation}
\end{small}
Note, that in this procedure we replace the notion of \emph{time} with \emph{layers}, and therefore instead of using $t$ to denote time, we use $\ell$ to denote the step or layer number. Here, $\epsilon$ is the discretization time step, which replaces the infinitesimal $dt$ from \cref{eq:new_adgn_graphwise}.
Furthermore, we note that similar to other DE-GNNs \cite{chamberlain2021grand, gravina2022anti} we use \emph{weight-sharing}, such that the learnable matrices $\bfW,\bfV, \bfZ$, as well as the spatial aggregation matrices $\hat{\bfA}, \tilde{\bfA}$ in $\Phi, \Psi$ (as shown in \cref{eq:aggFunc1,eq:aggFunc2}, and discussed in Appendix~\ref{app:learningSpatial}) are shared across layers.

Lastly, we notice that Euler's forward method is stable if $(1+\epsilon\lambda(\mathbf{J}(t)))$ resides within the unit circle in the complex plane for all eigenvalues of the system \cite{AscherPetzoldODEs}. However, given that the Jacobian matrix's eigenvalues are purely imaginary with our method, it implies $|1+\epsilon\lambda(\mathbf{J}(t))| > 1$. Consequently, Equations \eqref{eq:new_adgn_node} and \eqref{eq:new_adgn_graphwise} become unstable when solved using the forward Euler's method. To strengthen the stability of the numerical discretization method, similar to \cite{gravina2022anti} we introduce a small positive constant $\gamma > 0$, which is subtracted from the diagonal elements of the weight matrix $\mathbf{W}$, with the aim of placing back $(1+\epsilon\lambda(\mathbf{J}(t)))$ within the unit circle.

\subsection{Spatial Aggregation Terms}
\label{app:learningSpatial}
In \cref{eq:aggFunc1,eq:aggFunc2} we utilize two aggregation terms, denoted by $\hat{\bfA}$ and $\tilde{\bfA}$. The first, $\tilde{\bfA}$, is used to populate the \emph{weight} antisymmetric term $\Phi$, while the second, $\hat{\bfA}$, is used to populate the \emph{space} antisymmetric $\Psi$. As discussed, in our experiments we consider two possible parameterizations of these terms, on which we now elaborate.

\noindent \textbf{Pre-defined $\hat{\bfA}, \ \tilde{\bfA}$.} In this case, we denote our architecture as \emph{SWAN} and utilize the symmetric normalized adjacency matrix and the random walk normalized adjacency matrix for $\hat{\bfA}, \ \tilde{\bfA}$, respectively. Formally:
\begin{small}
\begin{equation}
    \label{eq:fixedOperators}
    \hat{\bfA} = \bfD^{-1/2}\bfA \bfD^{-1/2}, \quad \tilde{\bfA} = \bfD^{-1}\bfA,
\end{equation}
\end{small}
where $\bfA$ is the standard binary adjacency matrix that is induced by the graph connectivity $\mathcal{E}$, and $\bfD$ is the degree matrix of the graph.
We observe that the implementation of $\hat{\bfA}$ and $\tilde{\bfA}$ can be treated as a hyperparameter. To show this, in our experiments, we consider both the symmetric normalized adjacency matrix and the original adjacency matrix as implementations of $\hat{\bfA}$.

\noindent \textbf{Learnable $\hat{\bfA}, \ \tilde{\bfA}$.} Here, we denote our architecture by \emph{SWAN}-\textsc{learn}, and we use a multilayer perceptron (MLP) to learn edge-weights according to the original graph connectivity, to implement learnable $\hat{\bfA}, \ \tilde{\bfA}$. Specifically, we first define edge features as the concatenation of the initial embedding of input node features $\bfX^{(0)}$ of neighboring edges. Formally, the edge features of the $(u,v) \in \cal{E}$ edge, read: 
\begin{small}    
\begin{equation}
    \label{eq:edgeFeat}
    f_{(u,v)\in \cal{E}}^{in} = \bfx_u^{(0)} \oplus \bfx_v^{(0)}, \  f_{(u,v)\in \cal{E}}^{in} \in \mathbb{R}^{2d},
\end{equation}
\end{small}
where $\oplus$ denotes the channel-wise concatenation operator.
Then, we embed those features using a 2 layer MLP:
\begin{small}
\begin{equation}
    \label{eq:embedEdges}
    f_{(u,v)\in \cal{E}}^{emb} = {\rm{ReLU}}(\bfK_2 \sigma(\bfK_1(f_{((u,v)\in \cal{E})}^{in}))),
\end{equation}    
\end{small}
where $\bfK_1 \in \mathbb{R}^{d \times 2d}$ and $\bfK_2 \in \mathbb{R}^{n \times n}$ are learnable linear layer weights, and $\sigma$ is an activation function which is a hyperparameter of our method described in Appendix~\ref{app:hyperparams}. By averaging the feature dimension and gathering the averaged edge features $f_{(u,v)\in \cal{E}}^{emb}$ into a sparse matrix $\bfF \in \mathbb{R}^{|\mathcal{V}|\times|\mathcal{V}|}$, such that $\bfF_{u,v} = \frac{1}{d} \sum f_{(u,v)\in \cal{E}}^{emb}$ we define the learned spatial aggregation terms as follows:
\begin{small}    
\begin{equation}
    \label{eq:learnOperators}
    \hat{\bfA}_{\bfF} = \bfD_{\bfF}^{-1/2}\bfF \bfD_{\bfF}^{-1/2}, \quad \tilde{\bfA}_{\bfF} = \bfD_{\bfF}^{-1}\bfF,
\end{equation}
\end{small}
where $\bfD_{\bfF}$ is the degree matrix of $\bfF$, i.e., a matrix with the column sum of $\bfF$ on its diagonal and zeros elsewhere.

\subsection{The role of $\gamma$ and $\epsilon$}
The $\gamma$ term in SWAN refers to a small positive constant subtracted from the diagonal elements of the weight matrix to balance the stability of the numerical discretization method (i.e., Euler's forward method). According to \cite{AscherPetzoldODEs}, the Euler’s forward method applied to \cref{eq:new_adgn_nodewise} is stable when $|1+\epsilon\lambda(\mathbf{J}(t))| > 1$, with $\epsilon$ the step size. Since SWAN has $Re(\lambda(\mathbf{J}(t)))=0$, subtracting the constant $\gamma$ ensures that the Euler method remains stable. Therefore, as $\gamma$ approaches 0, the discretization method becomes increasingly unstable, requiring taking a smaller step size $\epsilon$.

\subsection{Applicability of SWAN to general MPNNs}\label{app:mpnn_applicability}
Nowadays, most GNNs rely on the concepts introduced by the Message Passing Neural Network (MPNN) \cite{gilmer2017neural}, which is a general framework based on the message-passing paradigm. 
A general MPNN updates the representation for a node $u$ by using message and update functions. The first (message function) is responsible for defining the messages between nodes and their neighbors. On the other hand, the update function has the role of collecting (aggregating) messages and updating the node representation. Our SWAN, in \cref{eq:new_adgn_nodewise}, operates according to the MPNN paradigm, with the functions $\Phi$ and $\Psi$ that operate as the message function, while the sum operator among the node and neighborhood representations is the update function. Therefore, SWAN can be interpreted as a special case of an MPNN with the aforementioned parameterization and can potentially be applied to different types of MPNNs. Lastly, we note that, although we implemented $\Phi$ and $\Psi$ as \cref{eq:aggFunc1,eq:aggFunc2}, their general formulation provides a flexible design principle for incorporating non-dissipativity as
an inductive bias into any DE-GNN, while remaining accessible for non-experts users. Specifically, the proposed antisymmetric aggregation function ($\Psi$) can be easily implemented by leveraging $\bfA - \bfA^\top$ as the antisymmetric graph shift operator, while the weight matrix can be symmetrized using $\bfZ + \bfZ^\top$, making it compatible with standard methods such as GCN and GAT.

\section{Proof of Theorem~\ref{thm:diffusionExpDecay}}\label{app:diffusionExpDecay}
The following proof follows the sensitivity analysis introduced in Section~\ref{sec:math}. 
\begin{proof}

Let us assume a diffusion-based network whose Jacobian's eigenvalues are represented by the diagonal matrix $\mathbf{\Lambda}$, and let us denote the eigenvalues magnitude by a diagonal matrix $\bfK\in\mathbb{R}_+^{n\times n}$ such that $\bfK_{ii} = | \mathbf{\Lambda}_{ii} |$ for $i \in \{0,\ldots, n-1\}$. Applying the vectorization operator, it is true that :
\begin{small}
\begin{equation}\label{eq:proofdiffusionExpDecay}
   \frac{\partial \rm{vec}(\mathbf{X}(t))}{\partial\rm{vec}(\mathbf{X}(0))} = e^{t \mathbf{J}} 
\end{equation}
\end{small}
As it is known from \cite{EvansPDE}, diffusion-based networks are characterized by Jacobian's eigenvalues with a negative real part. Indeed, diffusion DE-GNNs are based on the heat equation. Therefore, the right-hand side of the ODE is the graph Laplacian, reading $\frac{\partial\bfX(t)}{\partial t} = -\bfL\bfX(t)$. Therefore, following our derivations in Section \ref{sec:math}, we analyze $-\bfL$. It is known that the graph Laplacian has non-negative eigenvalues, and therefore $-\bfL$ has non-positive eigenvalues. Thus, the Jacobian has non-positive eigenvalues. Assuming a connected graph (\ie it exists at least one edge), then there is at least one value that is strictly negative in $-\bfL$, and any entry that is not non-negative will be equal to zero. Therefore, we can write $\bfJ= - \bfK$, leading to the equation
\begin{small}
\begin{equation}       
    \left\Vert\frac{\partial \rm{vec}(\mathbf{X}(t))}{\partial \rm{vec}(\mathbf{X}(0))}\right\Vert = \|e^{-t\bfK}\|. 
\end{equation}
\end{small}
Therefore, the information propagation rate among the graph nodes exponentially decays over time. Note that at $t \rightarrow \infty$, it will converge to 1, which is exponentially lower than the propagation rate at early time $t$.

\end{proof}

\section{Proof of Theorem~\ref{thm:swan_sensitivity}}\label{app:proof_sensitivity}
This proof follows the one proposed in \cite{diGiovanniOversquashing} (Appendix B). 
\begin{proof}
We proceed by induction and we show only the inductive step (\ie $\ell > 1$), since the case $\ell=1$ is straightforward (we refer to \cite{diGiovanniOversquashing} for more details). Assuming the Einstein summation convention and given that $\mathbf{\hat{W}} = \mathbf{W}- \mathbf{W}^\top$, $\mathbf{\hat{Z}} = \mathbf{Z}+ \mathbf{Z}^\top$, and $\alpha,\beta\in [p]$, we have:
\begin{small}
\begin{alignat*}{3}
\left|\frac{\partial \mathbf{x}_v^\alpha(\ell+1)}{\partial \mathbf{x}_u^\gamma(0)}\right| &= \left|\sigma^\prime\Bigl(c_r \mathbf{\hat{W}}^{(\ell)}_{\alpha, k} \frac{\partial \mathbf{x}_v^k(\ell)}{\partial \mathbf{x}_u^\gamma(0)} + c_a \mathbf{V}^{(\ell)}_{\alpha, k} \mathbf{A}_{vz} \frac{\partial \mathbf{x}_z^k(\ell)}{\partial \mathbf{x}_u^\gamma(0)} + \beta c_b \mathbf{\hat{Z}}^{(\ell)}_{\alpha, k} \mathbf{S}_{vz} \frac{\partial \mathbf{x}_z^k(\ell)}{\partial \mathbf{x}_u^\gamma(0)}\Bigr)\right|\\
 &\leq |\sigma^\prime|\Bigl(c_r |\mathbf{\hat{W}}^{(\ell)}_{\alpha, k}| \left|\frac{\partial \mathbf{x}_v^k(\ell)}{\partial \mathbf{x}_u^\gamma(0)}\right| + c_a |\mathbf{V}^{(\ell)}_{\alpha, k}| \mathbf{A}_{vz} \left|\frac{\partial \mathbf{x}_z^k(\ell)}{\partial \mathbf{x}_u^\gamma(0)}\right| + \beta c_b |\mathbf{\hat{Z}}^{(\ell)}_{\alpha, k}| \mathbf{S}_{vz} \left|\frac{\partial \mathbf{x}_z^k(\ell)}{\partial \mathbf{x}_u^\gamma(0)}\right|\Bigr)\\
 &\leq c_\sigma w\Bigl(c_r \left\Vert\frac{\partial \mathbf{x}_v(\ell)}{\partial \mathbf{x}_u(0)}\right\Vert + c_a \mathbf{A}_{vz} \left\Vert\frac{\partial \mathbf{x}_z(\ell)}{\partial \mathbf{x}_u(0)}\right\Vert + \beta c_b \mathbf{S}_{vz} \left\Vert\frac{\partial \mathbf{x}_z(\ell)}{\partial \mathbf{x}_u(0)}\right\Vert\Bigr)\\
    &\leq c_\sigma w (c_\sigma wp)^\ell\Bigl(c_r ((c_r \mathbf{I} + c_a \mathbf{A} + \beta c_b \mathbf{S})^{\ell})_{vu} + c_a \mathbf{A}_{vz} ((c_r \mathbf{I} + c_a \mathbf{A} + \beta c_b \mathbf{S})^{\ell})_{vz} + \\ &\hspace{5,7cm}+\beta c_b \mathbf{S}_{vz}((c_r \mathbf{I} + c_a \mathbf{A} + \beta c_b \mathbf{S})^{\ell})_{vz}\Bigr)\\
    &\leq c_\sigma w (c_\sigma wp)^\ell\Bigl((c_r \mathbf{I} + c_a \mathbf{A} + \beta c_b \mathbf{S})^{\ell+1}\Bigr)_{vu}
\end{alignat*}
\end{small}
where $|\cdot|$ denotes an absolute value of a real number, $w$ is the maximal entry-value over all weight matrices, and $c_\sigma$, $c_r$, $c_a$, and $c_b$ are the Lipschitz maps of the components in the computation of SWAN.
We can now sum over $\alpha$ on the left, generating an extra $p$ on the right side.
\end{proof}

\section{Datasets Details }\label{app:datasets}
\subsection{Graph Transfer Dataset} \label{app:datasets_graphtransfer}
We built the graph transfer datasets upon \cite{diGiovanniOversquashing}. In each task, graphs use identical topology, but, differently from the original work, nodes are initialized with random input features sampled from a uniform distribution in the interval $[0, 0.5)$. In each graph, we selected a source node and target node and initialized them with labels of value ``1'' and ``0'', respectively. We sampled graphs from three graph distributions, \ie line, ring, and crossed-ring. 
 Figure~\ref{fig:graph_transfer_example} shows a visual exemplification of the three types of graphs when the distance between the source and target nodes is 5. Specifically, ring graphs are cycles of size $n$, in which the target and source nodes are placed at a distance of $\lfloor n/2 \rfloor$ from each other. Crossed-ring graphs are also cycles of size $n$, but include crosses between intermediate nodes. Even in this case, the distance between source and target nodes remains $\lfloor n/2 \rfloor$. Lastly, the line graph contains a path of length $n$ between the source and target node. 
 In our experiments, we consider a regression task, whose aim is to swap source and target node labels while maintaining intermediate nodes unchanged. We use an input dimension of 1, and the distance between source and target nodes is equal to 3, 5, 10, and 50. We generated 1000 graphs for training, 100 for validation, and 100 for testing.

\begin{figure}[h]
\centering
\begin{subfigure}{0.29\textwidth}
    \includegraphics[width=\textwidth]{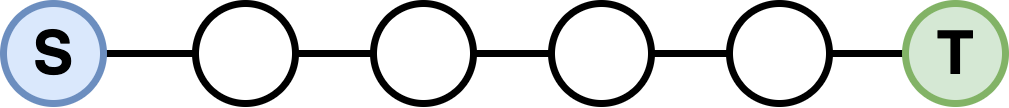}
    \vspace{0.12cm}
    \caption{Line}
\end{subfigure}
\hspace{0.8cm}
\begin{subfigure}{0.26\textwidth}
    \includegraphics[width=\textwidth]{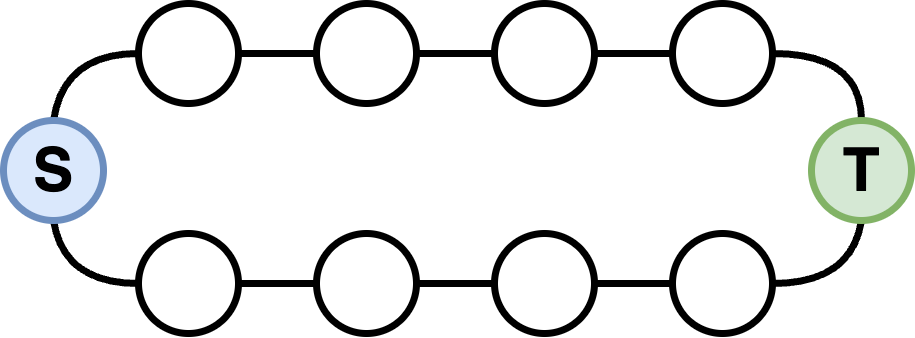}
    \caption{Ring}
\end{subfigure}
\hspace{0.8cm}
    \begin{subfigure}{0.26\textwidth}
    \includegraphics[width=\textwidth]{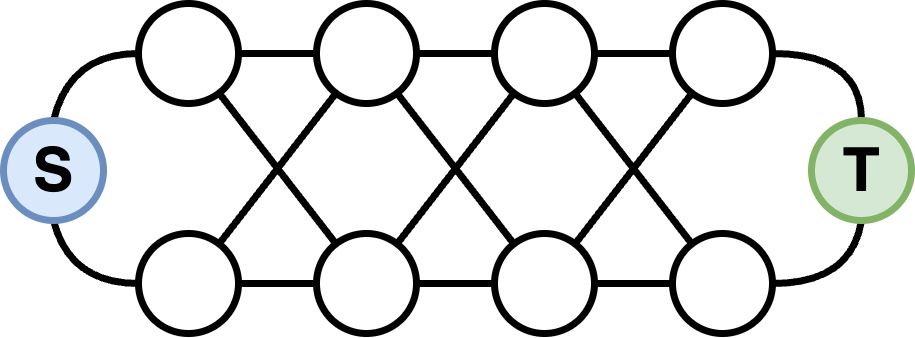}
\caption{Crossed-Ring}
\end{subfigure}
\caption{Line, ring, and crossed-ring graphs where the distance between source and target nodes is equal to 5. Nodes marked with ``S'' are source nodes, while the nodes with a ``T'' are target nodes.}
\label{fig:graph_transfer_example}
\end{figure}

\subsection{Graph Property Prediction}\label{app:datasets_graphproppred}
In our experiments on graph property prediction, we followed the data generation procedure outlined in \cite{gravina2022anti}. Graphs were randomly selected from various graph distributions, \ie Erd\H{o}s–R\'{e}nyi, Barabasi-Albert, grid, caveman, tree, ladder, line, star, caterpillar, and lobster. A visual examplification of each distribution is provided in \cref{fig:graph_distrib}. Each graph contained between 25 and 35 nodes, where each node is assigned with random identifiers as input features sampled from a uniform distribution in the interval $[0, 1)$. The target values represented single-source shortest paths, node eccentricity, and graph diameter. The dataset comprised a total of 7040 graphs, with 5120 used for training, 640 for validation, and 1280 for testing.

\begin{figure}[h]
\centering
\begin{subfigure}{0.2\textwidth}
    \includegraphics[width=1.1\textwidth]{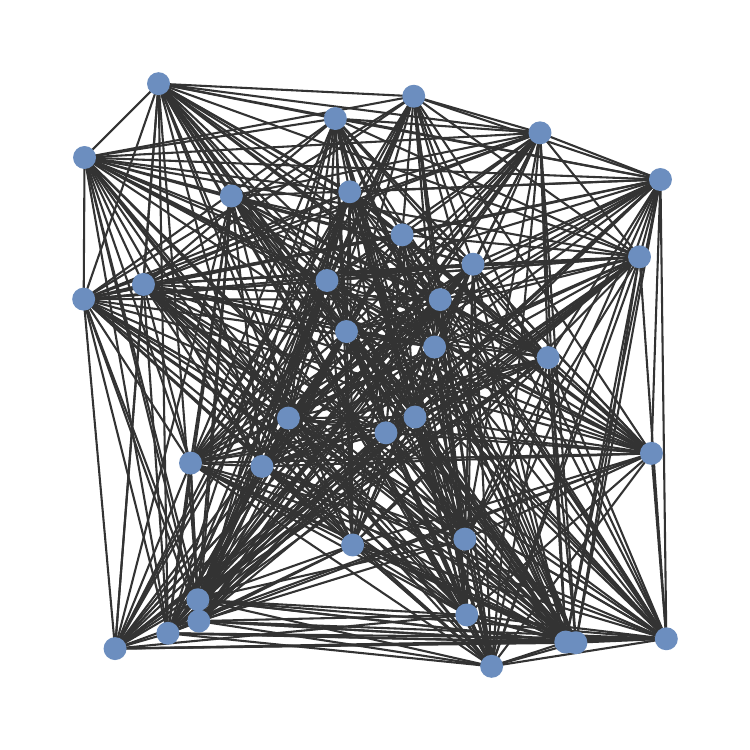}
    \caption{Erd\H{o}s–R\'{e}nyi}
\end{subfigure}
\begin{subfigure}{0.2\textwidth}
    \includegraphics[width=1.1\textwidth]{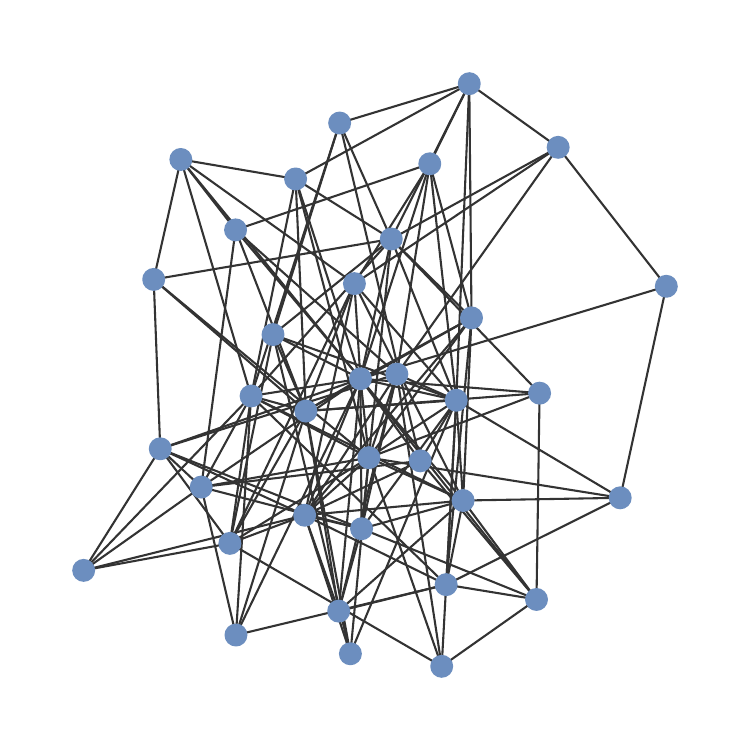}
    \caption{Barabasi-Albert}
\end{subfigure}
\begin{subfigure}{0.2\textwidth}
    \includegraphics[width=1.1\textwidth]{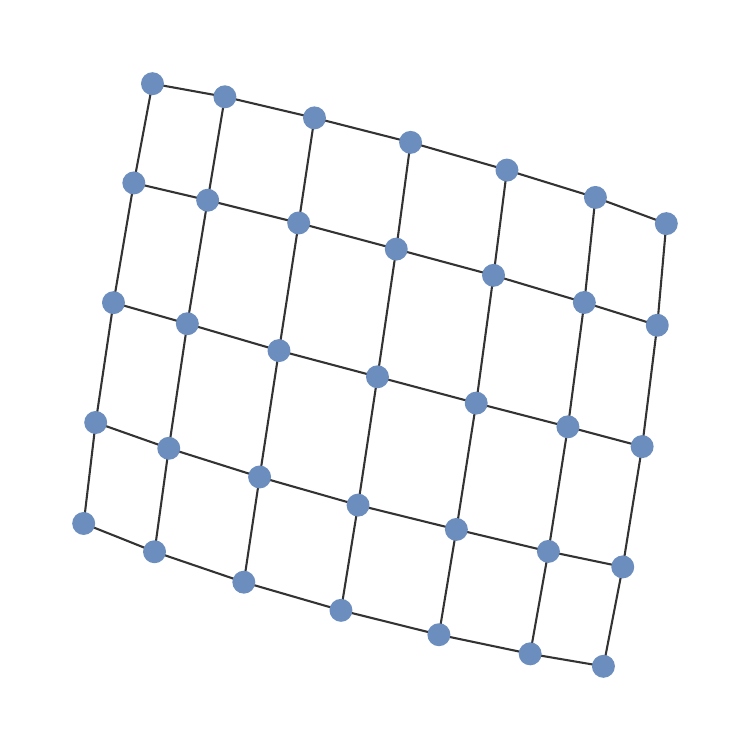}
    \caption{Grid}
\end{subfigure}
\begin{subfigure}{0.2\textwidth}
    \includegraphics[width=1.1\textwidth]{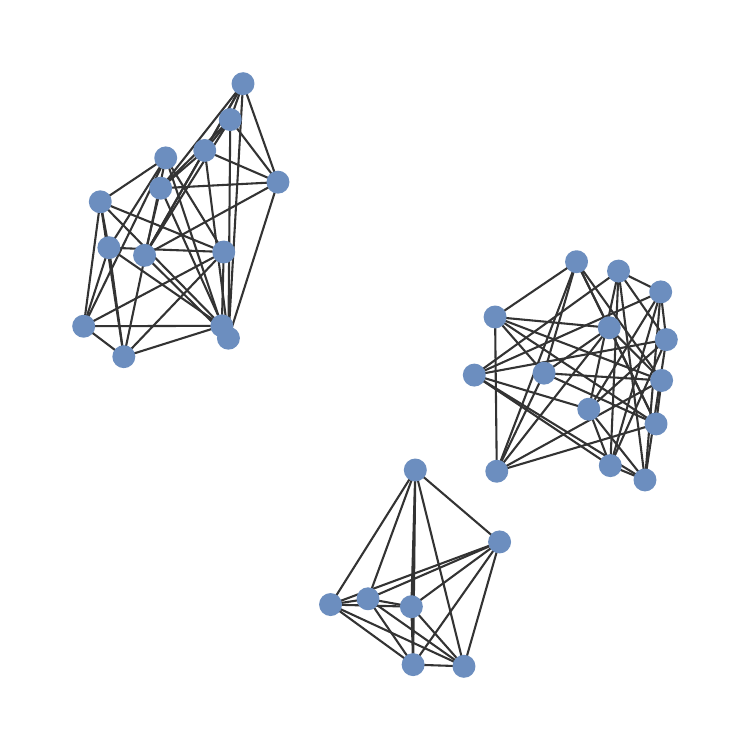}
    \caption{Caveman}
\end{subfigure}
\begin{subfigure}{0.2\textwidth}
    \includegraphics[width=1.1\textwidth]{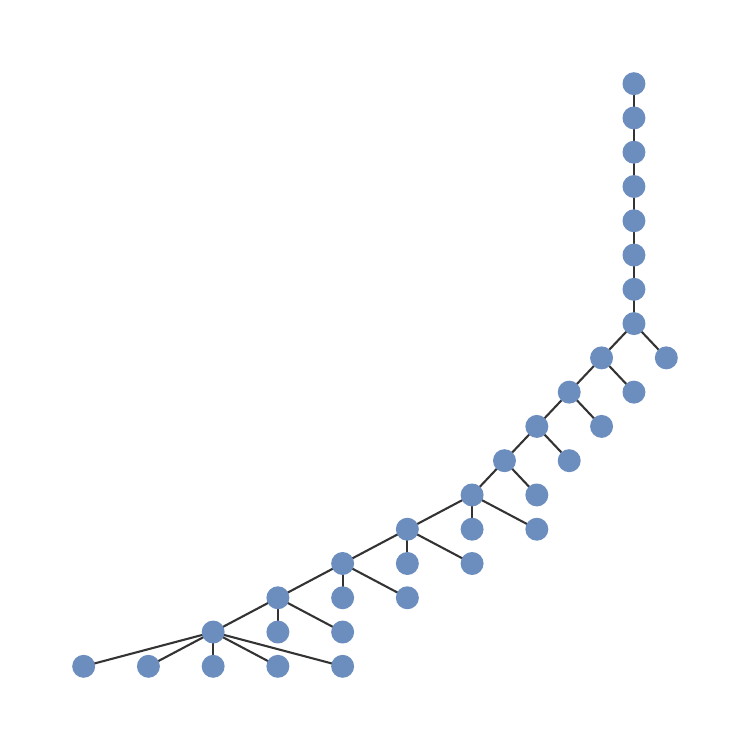}
    \caption{Tree}
\end{subfigure}
\begin{subfigure}{0.2\textwidth}
    \includegraphics[width=1.1\textwidth]{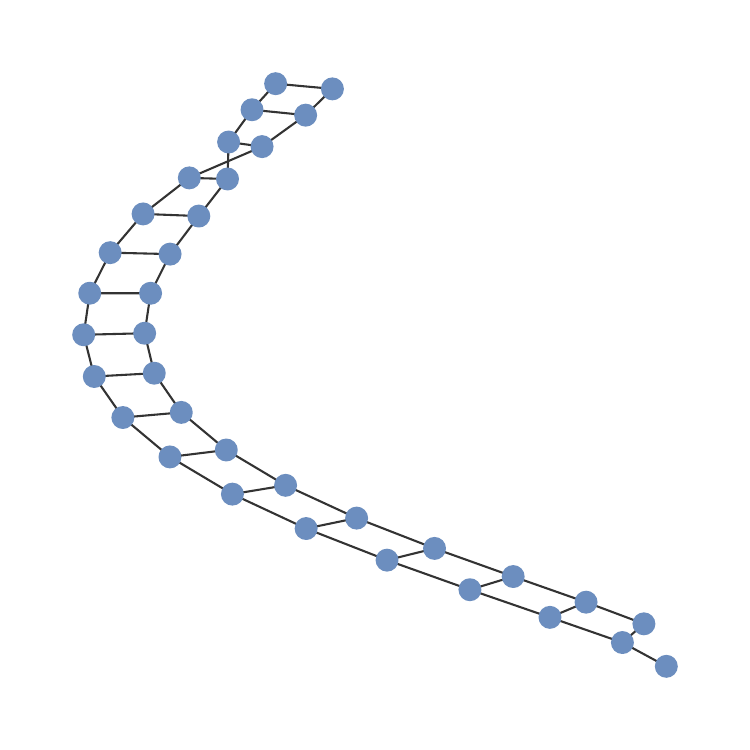}
    \caption{Ladder}
\end{subfigure}
\begin{subfigure}{0.2\textwidth}
    \includegraphics[width=1.1\textwidth]{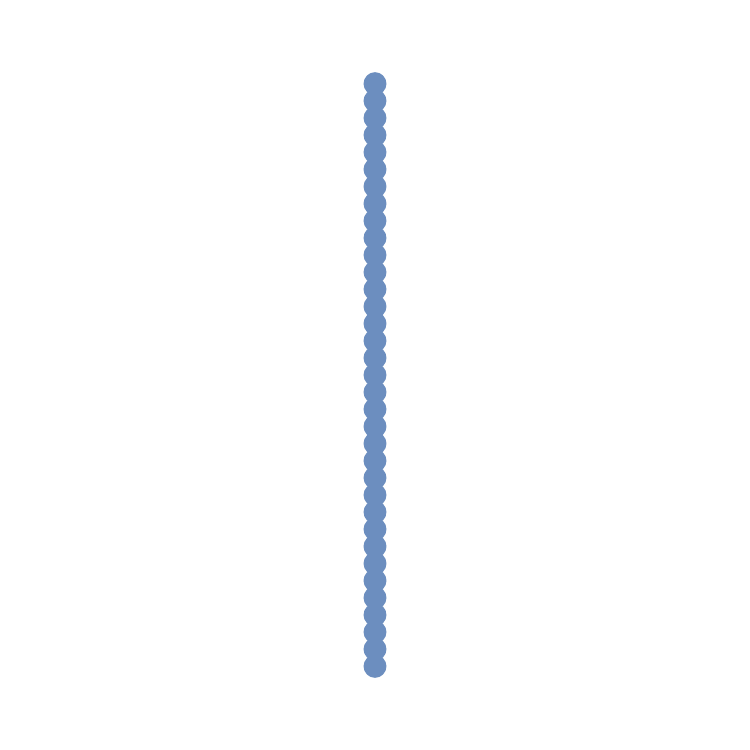}
    \caption{Line}
\end{subfigure}
\begin{subfigure}{0.2\textwidth}
    \includegraphics[width=1.1\textwidth]{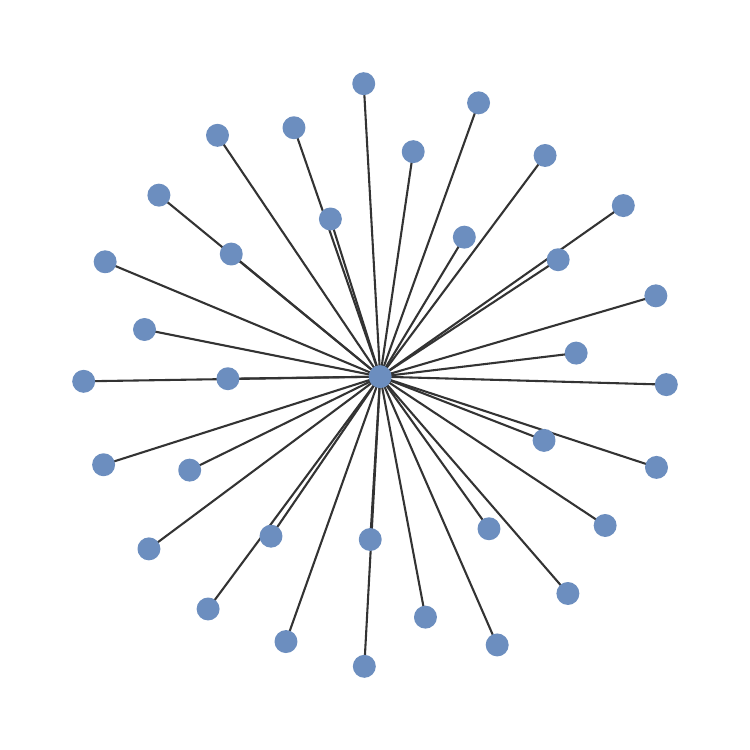}
    \caption{Star}
\end{subfigure}
\begin{subfigure}{0.2\textwidth}
    \includegraphics[width=1.1\textwidth]{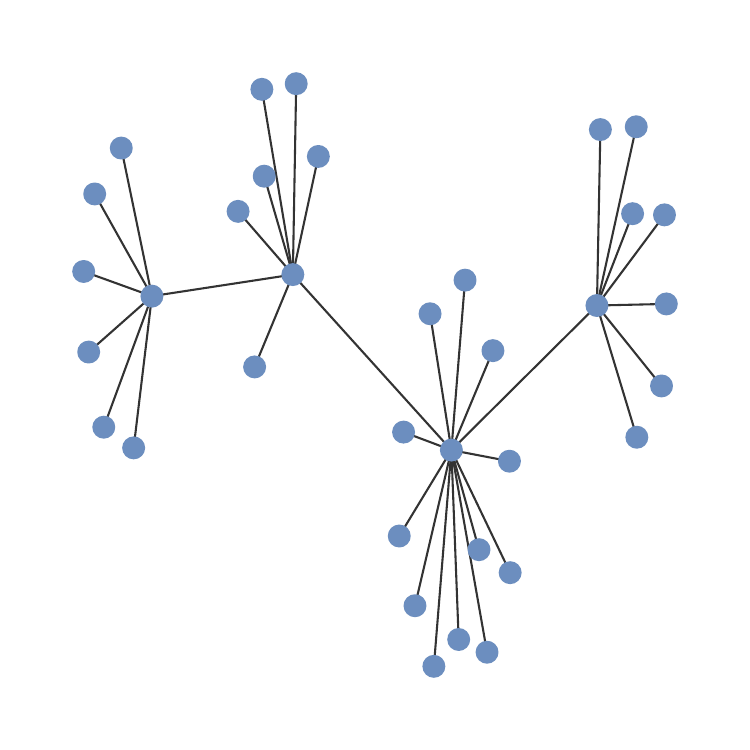}
    \caption{Caterpillar}
\end{subfigure}
\begin{subfigure}{0.2\textwidth}
    \includegraphics[width=1.1\textwidth]{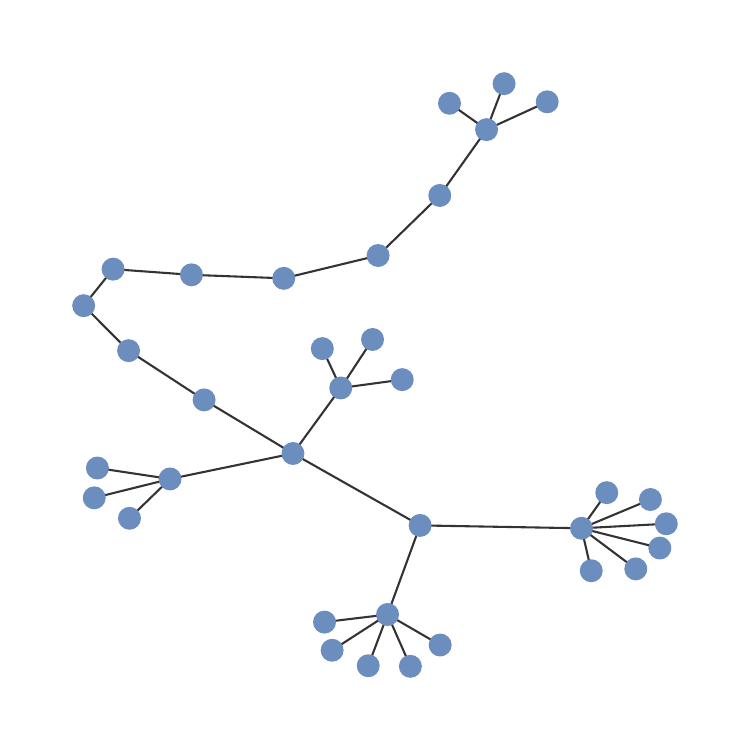}
    \caption{Lobster}
\end{subfigure}
\caption{Erd\H{o}s–R\'{e}nyi, Barabasi-Albert, grid, caveman, tree, ladder, line, star, caterpillar, and lobster graphs where the number of nodes is equal to 35.}
\label{fig:graph_distrib}
\end{figure}

\subsection{Long Range Graph Benchmark}\label{app:datasets_lrgb}
In the Long Range Graph Benchmark section, we considered the ``Peptides-func'', ``Peptides-struct'', and ``PascalVOC-sp'' datasets~\cite{dwivedi2022LRGB}. In the first two datasets, the graphs correspond to 1D amino acid chains, and they are derived such that the nodes correspond to the heavy (non-hydrogen) atoms of the peptides while the edges represent the bonds between them. Peptides-func is a multi-label graph classification dataset, with a total of 10 classes based on the peptide function, e.g., Antibacterial, Antiviral, cell-cell communication, and others. Peptides-struct is a multi-label graph regression dataset based on the 3D structure of the peptides. Specifically, the task consists of the prediction of the inertia of the molecules using the mass and valence of the atoms, the maximum distance between each atom-pairs, sphericity, and the average distance of all heavy atoms from the plane of best fit. Both Peptides-func and Peptides-struct consist of 15,535 graphs with a total of 2.3 million nodes. 
PascalVOC-sp is a node classification dataset composed of graphs created from the images in the Pascal VOC 2011 dataset~\citep{everingham2015pascal}. 
A graph is derived from each image by extracting superpixel nodes using the SLIC algorithm~\citep{achanta2012slic} and constructing a rag-boundary graph to interconnect these nodes.
Each node represents a region of an image that belongs to a specific class. The dataset consists of 11,355 graphs for a total of 5.4 million nodes. The task involves predicting the semantic segmentation label for each superpixel node across 21 different classes.

We applied stratified splitting to Peptides-func and Peptides-struct to generate balanced train–valid–test dataset splits, using the ratio of 70\%–15\%–15\%, respectively.
On PascalVOC-sp, we consider 8,498 graphs for training, 1,428 for validation, and 1,429 test.

\section{Experimental Details}
\label{app:experimental_details}
Our method is implemented in PyTorch. Our experimental results were obtained using NVIDIA RTX-3090 and A100  GPUs.
\subsection{SWAN versions}\label{app:versions}
In our experiments, we explore four variants of our SWAN method, depending on the implementation of $\Phi$ and $\Psi$, as reported in Table~\ref{tab:swan_versions}. These versions can be grouped into two main groups, which differ in the resulting non-dissipative behavior. In the first group, which contains SWAN and SWAN-\textsc{learn} both $\Phi$ and $\Psi$ lead to purely imaginary eigenvalues, since we use symmetric graph shift operators and antisymmetric weight matrices, thus allowing for a node- and graph-wise non-dissipative behavior. The last group, which includes SWAN-\textsc{ne} and SWAN-\textsc{learn-ne}, can deviate from being globally non-dissipative. Indeed, if the weight matrix, $\bfV$, is an arbitrary matrix, then the eigenvalues of the Jacobian matrix of the system, $\mathbf{J}(t)$, are contained in a neighborhood of the imaginary axis with radius $r \leq ||\mathbf{V}||$ (Bauer-Fike’s theorem~\cite{bauer1960norms}). Although this result does not guarantee that the eigenvalues of the Jacobian are imaginary, in practice, it crucially limits their position, limiting the dynamics of the system on the graph to show, at most, moderate amplification or loss of signals through the graph.

\begin{table}[h]
\centering
\footnotesize
\begin{tabular}{lll}
\hline\toprule
\textbf{Name} & $\Phi$ & $\Psi$\\\midrule

\textbf{Weight Antisymmetry Only} \\
$\,$ SWAN$_{\beta=0}$ & $(\hat{\bfA}+\hat{\bfA}^\top)\mathbf{X}(\mathbf{V}-\mathbf{V}^\top)$ & -- \\
\midrule
\textbf{Bounded Non-Dissipative}  \\
$\,$ SWAN-\textsc{ne} & $\hat{\bfA}\bfX\bfV$ & $(\tilde{\bfA}-\tilde{\bfA}^\top)\mathbf{X}(\mathbf{Z}+\mathbf{Z}^\top)$ \\
$\,$ SWAN-\textsc{learn-ne} & $\hat{\bfA}_\bfF\mathbf{X}\mathbf{V}$ & $(\mathbf{\tilde{A}}_\bfF-\mathbf{\tilde{A}}_\bfF^\top)\mathbf{X}(\mathbf{Z}+\mathbf{Z}^\top)$ \\

\midrule
\textbf{Global and Local Non-Dissipative}  \\ 

$\,$ SWAN & $(\hat{\bfA}+\hat{\bfA}^\top)\mathbf{X}(\mathbf{V}-\mathbf{V}^\top)$ & $(\tilde{\bfA}-\tilde{\bfA}^\top)\mathbf{X}(\mathbf{Z}+\mathbf{Z}^\top)$ \\
$\,$ SWAN-\textsc{learn} & $(\hat{\bfA}_\bfF+\hat{\bfA}_\bfF^\top)\mathbf{X}(\mathbf{V}-\mathbf{V}^\top)$ & $(\mathbf{\tilde{A}}_\bfF-\mathbf{\tilde{A}}_\bfF^\top)\mathbf{X}(\mathbf{Z}+\mathbf{Z}^\top)$ \\

\bottomrule\hline
\end{tabular}
\caption{The grid of the evaluated SWAN versions. We consider $\hat{\bfA}$ to be either the original adjacency matrix $\bfA$ or the symmetric normalized adjacency matrix, $\bfD^{-1/2}\bfA \bfD^{-1/2}$, while $\tilde{\bfA}$ is the random walk normalized adjacency matrix, $\bfD^{-1}\bfA$. The learned versions (-\textsc{learn}) employ learnable $\hat{\bfA}$ and $\tilde{\bfA}$ as described in Appendix~\ref{app:learningSpatial}, here referred as $\hat{\bfA}_\bfF$ and $\tilde{\bfA}_\bfF$.
\label{tab:swan_versions}
}
\end{table}

\subsection{Graph Transfer Task}
\label{app:graphTransfer_exp_details}

We design each model as a combination of three main components. The first is the encoder which maps the node input features into a latent hidden space; the second is the graph convolution (\ie SWAN or the other baselines); and the third is a readout that maps the output of the convolution into the output space. The encoder and the readout share the same architecture among all models in the experiments.

We perform hyperparameter tuning via grid search, optimizing the Mean Squared Error (MSE) computed on the node features of the whole graph. We train the models using Adam optimizer for a maximum of 2000 epochs and early stopping with patience of 100 epochs on the validation loss. For each model configuration, we perform 4 training runs with different weight initialization and report the average of the results.
We report in Appendix~\ref{app:hyperparams} the grid of hyperparameters exploited for this experiment.

\subsection{Graph Property Prediction}
\label{app:property_prediction_exp_details}

We employ the same datasets, hyperparameter space, and experimental setting presented in \cite{gravina2022anti}. Therefore, we perform hyperparameter tuning via grid search, optimizing the Mean Square Error (MSE), training the models using Adam optimizer for a maximum of 1500 epochs, and early stopping with patience of 100 epochs on the validation error. For each model configuration, we perform 4 training runs with different weight initialization and report the average of the results.
We report in Appendix~\ref{app:hyperparams} the grid of hyperparameters exploited for this experiment, while in Appendix~\ref{app:datasets} more details about the dataset.

\subsection{Long-Range Graph Benchmarks}
\label{app:lrgb_details}
We employ the same datasets and experimental setting presented in \cite{dwivedi2022LRGB}. Therefore, we perform hyperparameter tuning via grid search, optimizing the Average Precision (AP) in the Peptide-func task, the Mean Absolute Error (MAE) in the Peptide-struct task, and e macro weighted F1 score in PascalVOC-sp, training the models using AdamW optimizer for a maximum of 300 epochs. For each model configuration, we perform 3 training runs with different weight initialization and report the average of the results. Also, we follow the guidelines in \cite{dwivedi2022LRGB, drew} and stay within the 500K parameter budget.
We report in Appendix~\ref{app:hyperparams} the grid of hyperparameters exploited for this experiment, while in Appendix~\ref{app:datasets} more details about the dataset. 

\subsection{Hyperparameters}\label{app:hyperparams}
In Table~\ref{tab:hyperparams} we report the grids of hyperparameters employed in our experiments by each method. We recall that the hyperparameters $\epsilon$, $\gamma$, and $\hat{\bfA}$ refer to our method and ADGN, while $\beta$ only to our method. Moreover, we note that, in each graph transfer task, we use a number of layers that is
equal to the distance between the source and target nodes.
\begin{table}[h]
\centering

\footnotesize
\begin{tabular}{l|lll}
\hline\toprule
\multirow{2}{*}{\textbf{Hyperparameters}}  & \multicolumn{3}{c}{\textbf{Values}}\\\cmidrule{2-4}
& {\bf \emph{Transfer}} & {\bf \emph{GraphProp}} & {\bf \emph{LRGB}}\\\midrule
Optimizer       & Adam            & Adam                     & AdamW\\
Learning rate   & 0.001           & 0.003                    & 0.001, \ 0.0005\\
Weight decay    & 0              & $10^{-6}$                & 0, \ 0.0001\\
N. layers       & 3, 5, 10, 50    & 1, 5, 10, 20             & 5, \ 8, \ 16, \ 32\\
embedding dim   & 64              & 10, 20, 30               & 64, \ 128\\
$\hat{\bfA}$    & $\bfA$, $\bfD^{-1/2}\bfA \bfD^{-1/2}$ & $\bfA$, $\bfD^{-1/2}\bfA \bfD^{-1/2}$ & $\bfA$, $\bfD^{-1/2}\bfA \bfD^{-1/2}$\\ 
$\sigma$        & tanh            & tanh                     & tanh\\
$\epsilon$      & 0.5, 0.1        & 1, 0.1, 0.01             & 1, 0.01\\
$\gamma$        & 0.1             & 1, 0.1                   & 1, 0.1\\
$\beta$         & 1, 0.1, 0.01, -1           & 2, 1, 0.5, 0.1, -0.5, -1 & 1, 0.1\\
\bottomrule\hline
\end{tabular}
\caption{The grid of hyperparameters employed during model selection for the 
graph transfer tasks (\emph{Transfer}), graph property prediction tasks (\emph{GraphProp}), and Long Range Graph Benchmark (\emph{LRGB}). We observe that, in each graph transfer task, we use a number of layers that is equal to the distance between the source and target nodes.
 \label{tab:hyperparams}
}
\end{table}

\section{Complexity and Runtimes}
\label{app:complexity}
\textbf{Complexity Analysis.}
Our SWAN architecture remains within the computational complexity of MPNNs (e.g., \citet{kipf2016semi, xu2019how}) and other DE-GNNs such as GRAND \cite{chamberlain2021grand}. Specifically, each SWAN layer is linear in the number of nodes $|\cal{V}|$ and edges $|\cal{E}|$, therefore it has a time complexity of $\mathcal{O}(|\cal{V}|+|\cal{E}|)$. Assuming we compute $L$ steps of the ODE defined by SWAN, the overall complexity is $\mathcal{O} \left(L\cdot(|\cal{V}|+|\cal{E}|) \right)$.

\noindent \textbf{Runtimes.} We measure the training and inference (i.e., the epoch-time and test-set time) runtime of SWAN on the Peptides-struct dataset and compare it with other baselines. Our runtimes show that SWAN is able to obtain high performance, while retaining a linear complexity runtime that is parallel to other MPNNs and DE-GNNs. Our measurements are reported in Table \ref{tab:runtimes}, and presents the training and inference runtimes in seconds, as well as the obtained metric for that task, the mean-absolute-error (MAE), for reference. The runtimes were measured on an NVIDIA RTX-3090 GPU with 24GB of memory. In all time measurements, we use a batch size of 64, 128 feature channels, and 5 layers. For reference, we also provide the reported downstream task performance of each method.

\begin{table}[h]
    \centering
    \footnotesize
    
    \begin{tabular}{lcccc}
    \hline \toprule
         \textbf{Method} & \textbf{Training} & \textbf{Inference} & \textbf{MAE} {\small($\downarrow$)}  \\
         \midrule
         GCN &  2.90 & 0.32 &  0.3496$_{\pm0.0013}$  \\
         \midrule
         GraphGPS+LapPE & 23.04 & 2.39 &   0.2500$_{\pm0.0005}$ \\
         \midrule
         GraphCON & 3.03 & 0.27 & 0.2778$_{\pm0.0018}$\\
         ADGN & 2.83 & 0.25 & 0.2874$_{\pm0.0021}$\\
         \midrule
         SWAN & 2.88 & 0.24 & 0.2571$_{\pm0.0018}$\\
         SWAN-\textsc{learn} & 2.93 & 0.26  & 0.2485$_{\pm0.0009}$  \\
         \bottomrule \hline
    \end{tabular}
        \caption{Measured training and inference runtimes in seconds, and the obtained MAE of SWAN and other baselines on the Peptides-func dataset.\label{tab:runtimes}
        }
\end{table}

To further evaluate the scalability of our SWAN, we present in \cref{tab:additional_runtimes} results on larger datasets, \ie ogbn-arxiv \citep{hu2020ogb} (169,343 nodes; 1,166,243 edges), Questions \citep{platonov2023a} (48,921 nodes; 153,540 edges), Tolokers \citep{platonov2023a} (11,758 nodes; 519,000 edges). We report test accuracy and training/inference runtimes (ms) measured on an NVIDIA A100 with 40GB memory. Our results show that GraphGPS reaches out-of-memory (OOM) on larger datasets, while SWAN and GCN show similar runtimes, with SWAN outperforming GCN.

\begin{table}[h]
\centering
\footnotesize
    
\begin{tabular}{lccc|ccc|ccc}
\hline \toprule
&		\multicolumn{3}{c}{\textbf{ogbn-arxiv}} &			\multicolumn{3}{c}{\textbf{Questions}} & 			\multicolumn{3}{c}{\textbf{Tolokers}} \\		
\textbf{Method} &            \textbf{Training} & \textbf{Inference} & \textbf{Acc} ($\uparrow$) & \textbf{Training} & \textbf{Inference} & \textbf{Acc} ($\uparrow$) & 	\textbf{Training} & \textbf{Inference} & \textbf{Acc} ($\uparrow$)\\
         &   (ms)     & (ms)      &              & (ms)     &   (ms)    &                     & (ms)       &     (ms)    &   \\
\midrule
GCN               & 219.04  & 107.50    & 71.74$_{\pm2.90}$	& 86.04     &   41.22   &   76.09$_{\pm1.27}$	& 141.98	& 81.87	 & 83.64$_{\pm0.67}$	\\
GPS               & OOM     & OOM       & OOM	    & OOM       &   OOM	    &   OOM         & 1272.15	& 532.93 & 83.51$_{\pm0.93}$	\\\midrule
SWAN       & 228.45  & 116.06    & 72.98$_{\pm2.60}$ & 86.19     &   41.39   &   77.94$_{\pm1.16}$	& 142.60	& 82.09	 & 84.01$_{\pm0.54}$	\\
SWAN-\textsc{learn} & 255.84  & 131.48    & 73.26$_{\pm2.80}$ & 104.71    &   59.80   &   78.12$_{\pm1.28}$	& 167.04	& 97.98	 & 84.69$_{\pm0.63}$	\\
\bottomrule \hline
\end{tabular}
\caption{Measured training and inference runtimes (ms), and the obtained Accuracy of SWAN and other baselines on the ogbn-arxiv, Questions, and Tolokers datasets.
\label{tab:additional_runtimes}
}
\end{table}

\section{Additional Comparisons}
\label{app:additional_results}
\textbf{Graph Property Prediction.} Although we experimented SWAN with weight sharing to maintain consistency with DE-GNN literature \citep{chamberlain2021grand, rusch2022graph,eliasof2021pde,gravina2022anti}, a
more general version of the framework with layer-dependent weights, \ie $\bfW^{(\ell)} - (\bfW^{(\ell)})^\top$, $\bfV^{(\ell)}$, and $\bfZ^{(\ell)}$, is possible. We report in Table~\ref{tab:results_GraphProp_full} the comparison of SWAN and SWAN-\textsc{learn} with baseline methods employing both weight sharing and layer dependent weights configurations. Overall, the experiments demonstrate that while the weight-sharing version of SWAN already delivers consistent improvements over state-of-the-art methods, the introducing layer-dependent weights leads to further improvements, significantly outperforming all baselines. Lastly, we note that while there is occasional overlap in standard deviations between SWAN (employing weight sharing) and ADGN, (i) SWAN and SWAN-\textsc{learn} exhibit superior average performance, (ii) they significantly outperform ADGN on real-world long-range graph benchmarks (see \cref{sec:exp_lrb}), and (iii) the use of layer-dependent weights further amplifies the performance gap compared to ADGN.
\begin{table}[t]
\setlength{\tabcolsep}{1pt}
\centering
\footnotesize
\vspace{-2mm}
\begin{tabular}{lccc}
\hline\toprule
\textbf{Model} &\textbf{Diameter} & \textbf{SSSP} & \textbf{Eccentricity} \\\midrule
\textbf{MPNNs} \\
$\,$ GCN            & 0.7424$_{\pm0.0466}$ & 0.9499$_{\pm9.18\cdot10^{-5}}$ & 0.8468$_{\pm0.0028}$ \\
$\,$ GAT            & 0.8221$_{\pm0.0752}$ & 0.6951$_{\pm0.1499}$           & 0.7909$_{\pm0.0222}$  \\
$\,$ GraphSAGE      & 0.8645$_{\pm0.0401}$ & 0.2863$_{\pm0.1843}$           &  0.7863$_{\pm0.0207}$\\
$\,$ GIN            & 0.6131$_{\pm0.0990}$ & -0.5408$_{\pm0.4193}$          & 0.9504$_{\pm0.0007}$\\
$\,$  GCNII          & 0.5287$_{\pm0.0570}$ & -1.1329$_{\pm0.0135}$          & 0.7640$_{\pm0.0355}$\\
\midrule
\textbf{DE-GNNs} \\
$\,$ DGC            & 0.6028$_{\pm0.0050}$ & -0.1483$_{\pm0.0231}$          & 0.8261$_{\pm0.0032}$\\
$\,$ GRAND          & 0.6715$_{\pm0.0490}$ & -0.0942$_{\pm0.3897}$          & 0.6602$_{\pm0.1393}$ \\
$\,$ GraphCON       & 0.0964$_{\pm0.0620}$ & -1.3836$_{\pm0.0092}$ & 0.6833$_{\pm0.0074}$\\
$\,$ ADGN
& \three{-0.5188$_{\pm0.1812}$} & \three{-3.2417$_{\pm0.0751}$} & \three{0.4296$_{\pm0.1003}$}  \\

\midrule
\textbf{Ours - weight sharing} \\
SWAN & -0.5249$_{\pm0.0155}$ &  -3.2370$_{\pm0.0834}$ & 0.4094$_{\pm0.0764}$ \\
SWAN-\textsc{learn} & \two{-0.5981$_{\pm0.1145}$}  & -3.5425$_{\pm0.0830}$  & -0.0739$_{\pm0.2190}$ \\

\midrule
\textbf{Ours - layer dependent weights} \\
SWAN & \one{-0.6381$_{\pm0.0358}$} &  \one{-3.9342$_{\pm0.1993}$} & \one{-0.2706$_{\pm0.0948}$} \\
SWAN-\textsc{learn} & -0.5905$_{\pm0.0372}$  & \two{-3.8258$_{\pm0.0950}$} & \two{-0.2245$_{\pm0.0840}$} \\

\bottomrule\hline      
\end{tabular}
\caption{Mean test set {\small$log_{10}(\mathrm{MSE})$} and std averaged over 4 random weight initializations on the Graph Property Prediction tasks. The lower, the better. 
\one{First} and \two{second} best results and \three{best baseline} for each task are color-coded. 
\label{tab:results_GraphProp_full}
}
\end{table}

\noindent \textbf{Long Range Graph Benchmark.} In Table \ref{tab:lrgb_results} we are interested in directly comparing the performance of SWAN and its non-dissipative properties with other MPNNs and DE-GNNs, as well as the common approach of using graph transformer to address long-range interaction modeling. We now provide additional comparisons with other methods that utilize multi-hop information, and are therefore more computationally expensive than our SWAN, while also utilizing additional features such as the Laplacian positional encoding. In our evaluation of SWAN, we chose not to use additional feature enhancements, in order to provide a clear exposition of the contribution and importance of the local and global non-dissipativity offered by SWAN.  The additional comparisons are given in 
Table~\ref{tab:lrgb_results_appendix}. In the Peptide-struct task, SWAN demonstrates superior performance compared to all the compared methods. On the Peptide-func task,  SWAN outperforms MPNNs, Graph Transformers, DE-GNNs, and some multi-hop GNNs performance to other methods, while being second to multi-hop methods such as DRew. The PascalVOC-SP results show the effectiveness of SWAN with other DE-GNNs (GRAND, GraphCON, ADGN), which are the focus of our paper, while offering competitive results to multi-hop and transformer methods. However, it is essential to note that multi-hop GNNs incur higher complexity, while SWAN maintains a linear complexity. Therefore, we conclude that SWAN offers a highly effective approach for tasks that require long-range interactions, as in the LRGB benchmark.
\begin{table}[h]
\centering
\footnotesize
\setlength{\tabcolsep}{3.5pt}
\begin{tabular}{@{}lccc@{}}
\hline\toprule
\multirow{3}{*}{\textbf{Model}} & \textbf{Peptides-}  & \textbf{Peptides-}  & \textbf{Pascal}              \\
& \textbf{func} & \textbf{struct} & \textbf{voc-sp}
                               \\
                                & \scriptsize{AP $\uparrow$}                             & \scriptsize{MAE $\downarrow$} & \scriptsize{F1 $\uparrow$}                             \\ \midrule  
\textbf{MPNNs} \\
$\,$ GCN           & 59.30$_{\pm0.23}$ & 0.3496$_{\pm0.0013}$ & 12.68$_{\pm0.60}$\\
$\,$ GINE          & 54.98$_{\pm0.79}$ & 0.3547$_{\pm0.0045}$ & 12.65$_{\pm0.76}$\\
$\,$ GCNII         & 55.43$_{\pm0.78}$ & 0.3471$_{\pm0.0010}$ & 16.98$_{\pm0.80}$\\
$\,$ GatedGCN      & 58.64$_{\pm0.77}$ & 0.3420$_{\pm0.0013}$ & 28.73$_{\pm2.19}$\\
$\,$ GatedGCN+PE   & 60.69$_{\pm0.35}$ & 0.3357$_{\pm0.0006}$ & 28.60$_{\pm0.85}$\\ 
\midrule
\textbf{Multi-hop GNNs}\\
$\,$ DIGL+MPNN           & 64.69$_{\pm0.19}$         & 0.3173$_{\pm0.0007}$ & 28.24$_{\pm0.39}$\\
$\,$ DIGL+MPNN+LapPE     & 68.30$_{\pm0.26}$         & 0.2616$_{\pm0.0018}$ & 29.21$_{\pm0.38}$\\
$\,$ MixHop-GCN          & 65.92$_{\pm0.36}$         & 0.2921$_{\pm0.0023}$ & 25.06$_{\pm1.33}$\\
$\,$ MixHop-GCN+LapPE    & 68.43$_{\pm0.49}$         & 0.2614$_{\pm0.0023}$ & 22.18$_{\pm1.74}$\\ 
$\,$ DRew-GCN            & \three{69.96$_{\pm0.76}$} & 0.2781$_{\pm0.0028}$ & 18.48$_{\pm1.07}$\\
$\,$ DRew-GCN+LapPE      & \one{71.50$_{\pm0.44}$}   & 0.2536$_{\pm0.0015}$ & 18.51$_{\pm0.92}$\\
$\,$ DRew-GIN            & 69.40$_{\pm0.74}$         & 0.2799$_{\pm0.0016}$ & 27.19$_{\pm0.43}$\\
$\,$ DRew-GIN+LapPE      & \two{71.26$_{\pm0.45}$}   & 0.2606$_{\pm0.0014}$ & 26.92$_{\pm0.59}$\\
$\,$ DRew-GatedGCN       & 67.33$_{\pm0.94}$         & 0.2699$_{\pm0.0018}$ & 32.14$_{\pm0.21}$\\
$\,$ DRew-GatedGCN+LapPE & 69.77$_{\pm0.26}$         & 0.2539$_{\pm0.0007}$ & \two{33.14$_{\pm0.24}$}\\


\midrule
\textbf{Transformers} \\
$\,$ Transformer+LapPE & 63.26$_{\pm1.26}$ & \three{0.2529$_{\pm0.0016}$} & 26.94$_{\pm0.98}$\\
$\,$ SAN+LapPE         & 63.84$_{\pm1.21}$ & 0.2683$_{\pm0.0043}$         & \three{32.30$_{\pm0.39}$}\\
$\,$ GraphGPS+LapPE    & 65.35$_{\pm0.41}$ & \two{0.2500$_{\pm0.0005}$}   & \one{37.48$_{\pm1.09}$}\\ 
\midrule
\textbf{DE-GNNs} \\
$\,$ GRAND     & 57.89$_{\pm0.62}$ & 0.3418$_{\pm0.0015}$ &  19.18$_{\pm0.97}$ \\
$\,$  GraphCON & 60.22$_{\pm0.68}$ & 0.2778$_{\pm0.0018}$ &  21.08$_{\pm0.91}$ \\
$\,$ ADGN      & 59.75$_{\pm0.44}$ & 0.2874$_{\pm0.0021}$ &  23.49$_{\pm0.54}$\\ 
\midrule
\textbf{Ours} \\    
$\,$ SWAN                & 63.13$_{\pm0.46}$ & 0.2571$_{\pm0.0018}$       & 27.96$_{\pm0.48}$\\
$\,$ SWAN-\textsc{learn} & 67.51$_{\pm0.39}$ & \one{0.2485$_{\pm0.0009}$} & 31.92$_{\pm2.50}$\\

\bottomrule\hline
\end{tabular}
\caption{Performance of various classical, multi-hop and static rewiring MPNN, graph Transformer benchmarks, DE-GNNs, and our SWAN across two LRGB tasks. Results are averaged over 3 weight initializations. The \one{first}, \two{second}, and \three{third} best results for each task are color-coded. Beseline results are reported from \cite{drew}.
\label{tab:lrgb_results_appendix}
}
\end{table}

\section{Derivation of the Graph-wise Jacobian}
\label{app:jacobian}
Recall the ODE that defines SWAN  in \cref{eq:new_adgn_graphwise}:
\begin{small}
\begin{equation}
 \label{eq:new_adgn_graphwise_APP}
     \frac{\partial\mathbf{X}(t)}{\partial t} = \sigma \Bigl(\mathbf{X}(t){(\mathbf{W}-\mathbf{W}^\top)} 
     +{(\hat{\mathbf{A}}+\hat{\mathbf{A}}^\top)\mathbf{X}(t)(\mathbf{V}-\mathbf{V}^\top)}
     + \beta{(\tilde{\mathbf{{A}}}-\tilde{\mathbf{{A}}}^\top)\mathbf{X}(t)(\mathbf{Z}+\mathbf{Z}^\top)}\Bigr).
 \end{equation}
\end{small}

 The Jacobian of \cref{eq:new_adgn_graphwise_APP} with respect to $\mathbf{X}(t)$ is:
 \begin{small}
 \begin{align}
 \mathbf{M}_1 &= \sigma'\left(\mathbf{X}(t){(\mathbf{W}-\mathbf{W}^\top)} 
     +{(\hat{\mathbf{A}}+\hat{\mathbf{A}}^\top)\mathbf{X}(t)(\mathbf{V}-\mathbf{V}^\top)}
     + \beta{(\tilde{\mathbf{{A}}}-\tilde{\mathbf{{A}}}^\top)\mathbf{X}(t)(\mathbf{Z}+\mathbf{Z}^\top)} \right)\\
 \mathbf{M}_2 &=
     {(\mathbf{W}-\mathbf{W}^\top)}  
     +{(\mathbf{V}-\mathbf{V}^\top)^\top \otimes (\hat{\mathbf{A}}+\hat{\mathbf{A}}^\top) }
     + \beta(\mathbf{Z}+\mathbf{Z}^\top)^\top \otimes  (\tilde{\mathbf{{A}}}-\tilde{\mathbf{{A}}}^\top)
 \end{align}
 \end{small}
 where $\otimes$ is the Kronecker product.
 To analyze $\textbf{M}_1$, we use the following identity for arbitrary matrices $\mathbf{A}, \mathbf{X}, \mathbf{B}$ with appropriate dimensions (i.e., not related to the notations of this paper): 
 \begin{small}
 \begin{equation}    
 \label{eq:identity}
 \rm{vec}(\mathbf{AXB}) = \left( \mathbf{B}^\top \otimes \mathbf{A} \right) {\rm{vec}}(\mathbf{X})
 \end{equation}
 \end{small}
 Using the identity from \cref{eq:identity}, we can rewrite $\mathbf{M}_1$ as follows:
 \begin{small}
 \begin{align}
\nonumber \mathbf{M}_1 &= \sigma'\Bigl(\mathbf{I}\mathbf{X} (\mathbf{W}-\mathbf{W}^\top) +  (\hat{\mathbf{A}}+\hat{\mathbf{A}}^\top)\mathbf{X}(t)(\mathbf{V}-\mathbf{V}^\top)  +\beta(\tilde{\mathbf{{A}}}-\tilde{\mathbf{{A}}}^\top)\mathbf{X}(t)(\mathbf{Z}+\mathbf{Z}^\top) \Bigr) \\
\nonumber & ={\rm{diag}}\left(\rm{vec}(\sigma'\Bigl(\mathbf{I}\mathbf{X} (\mathbf{W}-\mathbf{W}^\top) +  (\hat{\mathbf{A}}+\hat{\mathbf{A}}^\top)\mathbf{X}(t)(\mathbf{V}-\mathbf{V}^\top)  +\beta(\tilde{\mathbf{{A}}}-\tilde{\mathbf{{A}}}^\top)\mathbf{X}(t)(\mathbf{Z}+\mathbf{Z}^\top) ))\right) \\
\nonumber &={\rm{diag}}\Bigl(\sigma'\Bigl(((\mathbf{W}-\mathbf{W}^\top)^\top \otimes \mathbf{I)} \rm{vec}(\mathbf{X})  +\\\nonumber&\hspace{3cm}+ ((\mathbf{V}-\mathbf{V}^\top)^\top \otimes (\hat{\mathbf{A}}+\hat{\mathbf{A}}^\top)) \rm{vec}(\mathbf{X}(t)) +\\&\hspace{3cm}+\beta ((\mathbf{Z}+\mathbf{Z}^\top)^\top \otimes (\tilde{\mathbf{{A}}}-\tilde{\mathbf{{A}}}^\top)) \rm{vec}(\mathbf{X}(t)) \Bigr)\Bigr),
 \end{align}
 \end{small}
where $\mathbf{I}$ is the identity matrix. Therefore $\mathbf{M}_1$ is a diagonal matrix.

We note that Equation~\ref{eq:new_adgn_graphwise_APP} is the result of the composite function $\sigma(g(\mathbf{X}(t)))$, where $\sigma$ is the activation function and $g(\mathbf{X}(t))=\mathbf{X}(t){(\mathbf{W}-\mathbf{W}^\top)} +{(\hat{\mathbf{A}}+\hat{\mathbf{A}}^\top)\mathbf{X}(t)(\mathbf{V}-\mathbf{V}^\top)} + \beta{(\tilde{\mathbf{{A}}}-\tilde{\mathbf{{A}}}^\top)\mathbf{X}(t)(\mathbf{Z}+\mathbf{Z}^\top)}$.

Therefore, $\textbf{M}_2$ results from the derivative of $g$ with respect to $\mathbf{X}(t)$. Considering the $\rm{vec}$ operator, we have
\begin{small}
\begin{align}
\nonumber\textbf{M}_2 &= \rm{vec}(g'(\mathbf{X}(t)))\\
\nonumber&= g'(\rm{vec}(\mathbf{X}(t)))\\
&=
     {(\mathbf{W}-\mathbf{W}^\top)}  
     +{(\mathbf{V}-\mathbf{V}^\top)^\top \otimes (\hat{\mathbf{A}}+\hat{\mathbf{A}}^\top) }
     + \beta(\mathbf{Z}+\mathbf{Z}^\top)^\top \otimes  (\tilde{\mathbf{{A}}}-\tilde{\mathbf{{A}}}^\top)
\end{align}
\end{small}

\section{The Stability of the Jacobian}\label{appendix:jacobian_stable}
As discussed in Section \ref{sec:math}, assuming that the Jacobian of the underlying system does not change significantly over time allows us to analyze the system from an autonomous system perspective~\cite{AscherPetzoldODEs} 
and mirrors prior approaches \cite{RuthottoHaber2018, chen2018neural,chang2018antisymmetricrnn, chamberlain2021grand, gravina2022anti}. In addition to building on existing literature, below, we provide an empirical measurement on a real-world dataset (peptides-func) of the Jacobian of our SWAN over time (layers). For reference, we compare it with the Jacobian of GCN. As can be seen from Figure \ref{fig:jacobian_over_time}, the Jacobian of SWAN has a minimal Jacobian change over time with an average of 0.6\% between layers, while the change in the Jacobian over time in GCN is 40\% on average.

\begin{figure}[h]
\centering
\begin{adjustbox}{angle=90}
\begin{minipage}{.001\textwidth}
\tiny
    \begin{equation*}
\hspace{-3mm}
\frac{|\mathbf{J}^{(\ell)}-\mathbf{J}^{(\ell-1)}|}{|\mathbf{J}^{(\ell-1)}|}
    \end{equation*}
  \end{minipage}%
    \end{adjustbox}
\begin{minipage}{.42\textwidth}
    \centering
    \includegraphics[width=\textwidth]{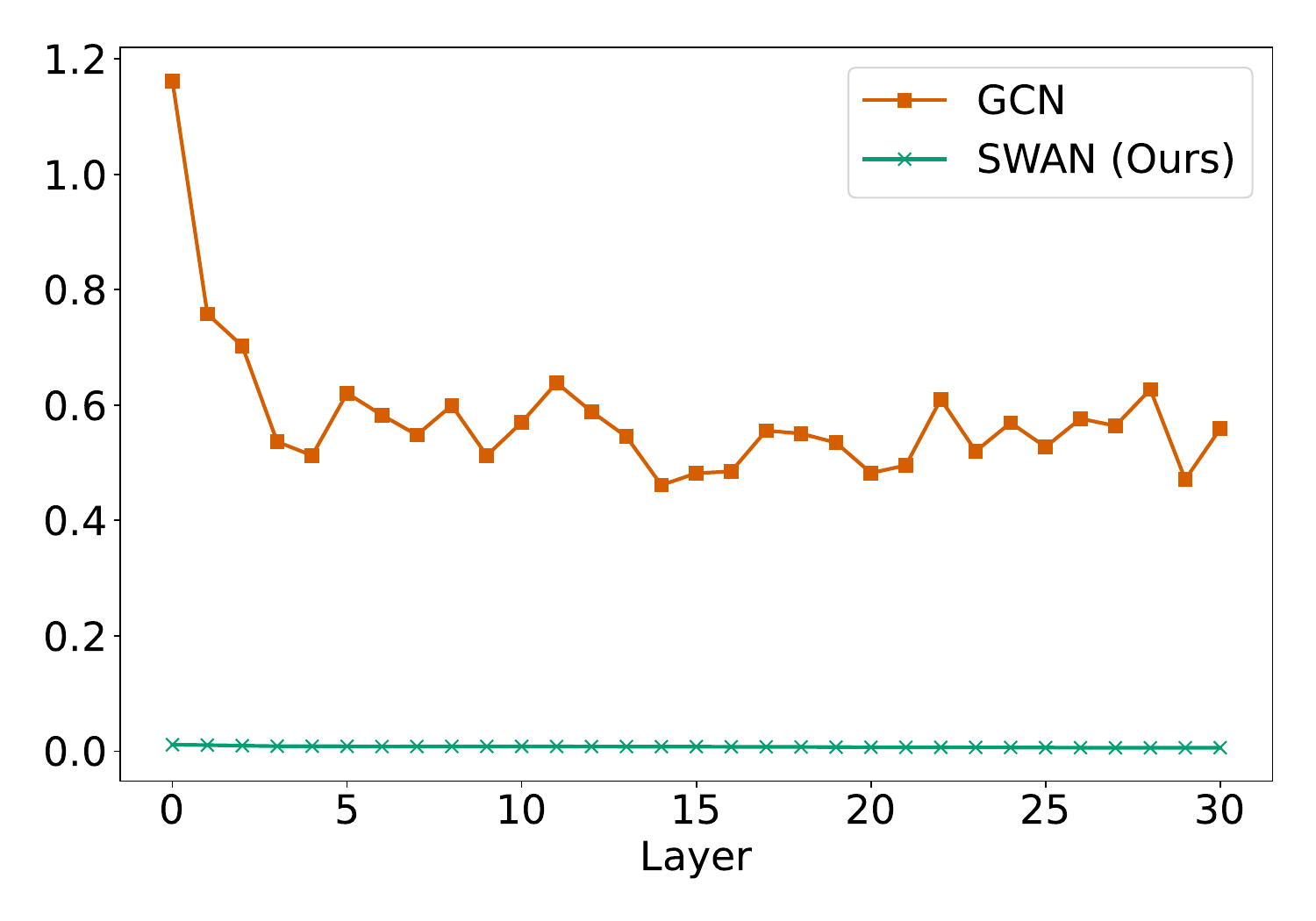}
    \caption{Arithmetic scale}
\end{minipage}
\begin{adjustbox}{angle=90}
\begin{minipage}{.001\textwidth}
    \tiny
    \begin{equation*}
    \hspace{-3mm}
\frac{|\mathbf{J}^{(\ell)}-\mathbf{J}^{(\ell-1)}|}{|\mathbf{J}^{(\ell-1)}|}
    \end{equation*}
  \end{minipage}%
\end{adjustbox}
\begin{minipage}{.42\textwidth}
    \centering
    \includegraphics[width=\textwidth]{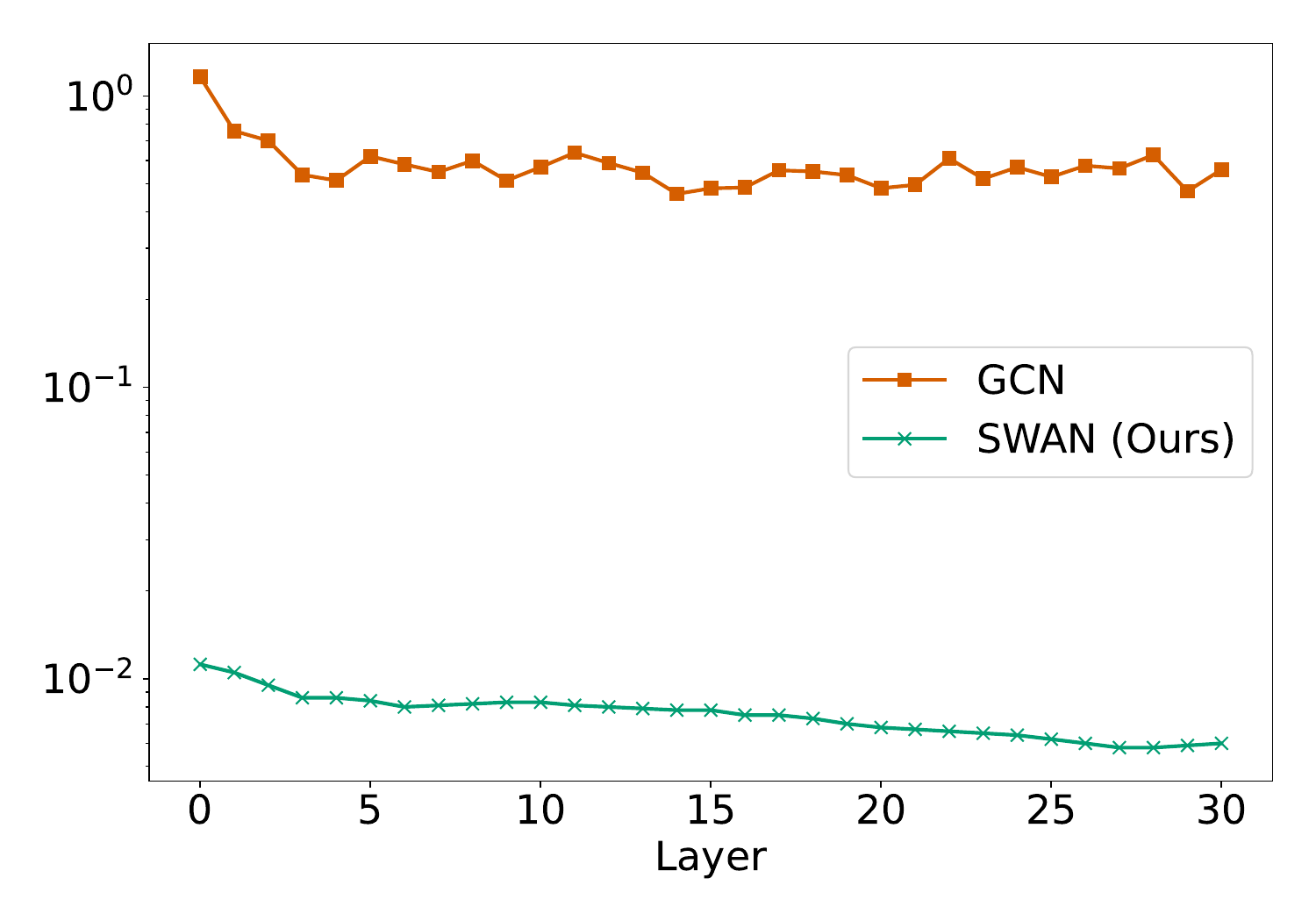}
    \caption{Logarithmic scale}
\end{minipage}
\caption{The relative change in Jacobian of GCN and our SWAN over layers, measured on the peptides-func data.}
\label{fig:jacobian_over_time}
\end{figure}

\section{Limitations and Broader Impact}\label{app:limits_and_impacts}
This paper aims to contribute to the field of Machine Learning, specifically focusing on advancing Differential-Equation inspired Graph Neural Networks (DE-GNNs). The research presented herein has a positive impact on the ongoing exploration and applications of GNNs, with an emphasis on applications that require long-range propagation or suffer from oversquashing. Potentially, it might be used in other aspects of GNNs.

Although we believe that our work has a positive impact, it is important to note that SWAN, as our title suggests, aims to propose a mechanism to address the oversquashing problem in GNNs through the perspective of DE-GNNs, and it is therefore focused on problems that require long interaction modeling. Also, SWAN provides an understanding of the model and its expected behavior, while keeping low computational costs (as shown and analyzed in Appendix~\ref{app:complexity}), and also offers significant improvement in its class, as discussed in the paper. However, if one aims to achieve the highest downstream performance, regardless of the computational complexity and the theoretical understanding, other solutions, such as higher-order GNNs or transformers, may also be viable and achieve higher performance at times.

\end{document}